\documentclass{article}
\usepackage{graphicx} 
\usepackage[main, final]{neurips_2026}
\usepackage{multirow} 
\usepackage{booktabs}
\usepackage{wrapfig}
\usepackage{tikz}
\usetikzlibrary{
  arrows.meta, decorations.pathmorphing, decorations.pathreplacing,
  positioning, shapes.geometric, shapes.misc, calc, backgrounds, hobby, fit
}

\tikzset{
  arr/.style    = {-{Stealth[length=6pt,width=4pt]}, line width=1.5pt},
  darr/.style   = {-{Stealth[length=5pt,width=3pt]}, line width=1.0pt, dashed},
  pbox/.style   = {rectangle, draw=black, fill=gray!15, thick,
                   minimum size=1.0cm, rounded corners=2pt},
  vbox/.style   = {rectangle, draw=black, fill=gray!10,
                   thick, rounded corners=3pt,
                   minimum width=1.1cm, minimum height=0.6cm},
  outbox/.style = {rectangle, draw=black, fill=white, thick,
                   rounded corners=3pt, minimum width=0.85cm,
                   minimum height=0.85cm, font=\scriptsize},
  flowbox/.style = {rectangle, draw=black!50, fill=gray!5,
                    thick, rounded corners=4pt,
                    inner sep=6pt,
                    dash pattern=on 3pt off 2pt},
}

\newcommand{\noisygridicon}[2]{%
  \foreach \col/\row/\op in {#1}{%
    \fill[gray!\op!white, draw=black!40, thin]
      ({-0.18+\col*0.19},{0.01-\row*0.19}) rectangle ++(0.17,0.17);
  }%
  \draw[black!70, rounded corners=2pt, line width=0.9pt]
    (-0.20,-0.20) rectangle (0.20,0.20);
  \foreach \nx/\ny in {#2}{%
    \fill[black!60] (\nx,\ny) circle (1.4pt);
  }%
  \draw[black!50, dashed, rounded corners=2pt, line width=0.75pt]
    (-0.26,-0.26) rectangle (0.26,0.26);
}

\usepackage{amsmath,amssymb,amstext,amsthm}
\DeclareMathOperator*{\argmin}{arg\,min}
\DeclareMathOperator*{\argmax}{arg\,max}
\usepackage{mathtools}
\usepackage{bm}
\usepackage[dvipsnames]{xcolor}
\usepackage{hyperref}
\hypersetup{
    colorlinks=true,
    linkcolor=blue!60!black,
    urlcolor=blue!60!black,
    citecolor=blue!60!black
}
 
\theoremstyle{plain}
\newtheorem{theorem}{Theorem}
\newtheorem{corollary}[theorem]{Corollary}
\newtheorem{lemma}[theorem]{Lemma}
 \newtheorem{proposition}[theorem]{Proposition}

\theoremstyle{definition}

\theoremstyle{remark}
\newtheorem{remark}[theorem]{Remark}
 
\newcommand{\R}{\mathbb{R}}
\newcommand{\1}{\mathbf{1}}
\newcommand{\C}{\mathcal{C}}
\newcommand{\B}{\mathcal{B}}
\newcommand{\im}{\operatorname{Im}}
\newcommand{\tr}{\operatorname{tr}}

\newcommand{\inner}[2]{\langle #1, #2 \rangle_F}

\usepackage{algpseudocode}
\usepackage{algorithm}
\title{Learning Unbiased Permutations via Flow Matching}
\author{%
  Yimeng Min\thanks{Currently at the Department of Computer Science
  and Engineering, Texas A\&M University, College Station, TX, USA.} \\
  Department of Computer Science\\
  Cornell University \\
  Ithaca, NY, USA \\
  \texttt{min@cs.cornell.edu} \\
  \And
  Carla P. Gomes\\
  Department of Computer Science\\
  Cornell University \\
  Ithaca, NY, USA \\
  \texttt{gomes@cs.cornell.edu}\\
}

\begin{document}

\maketitle
\begin{abstract}
Learning permutations is fundamental to sorting, ranking, and matching, but existing differentiable methods based on entropy-regularized Sinkhorn produce a single softened solution and collapse under ambiguity. We present PermFlow, a conditional flow matching framework that operates directly on the affine subspace of matrices with unit row and column sums. A closed-form tangent-space projector preserves these constraints exactly along every trajectory, by construction rather than through iterative correction, and a nearest-target coupling routes distinct noisy initializations toward distinct valid permutations. The result is a model that captures multimodal permutation distributions rather than collapsing them to a single mode. On a range of problems, PermFlow achieves high accuracy on unambiguous inputs and recovers both valid permutations under ambiguity.

\end{abstract}
\section{Introduction}
Flow-based generative models have emerged as a powerful generative modeling paradigm, achieving widespread adoption across diverse domains such as computer vision and natural language processing. In these settings, the data lie in continuous spaces (e.g., high-dimensional vectors), where flows offer a principled mechanism for transforming simple noise into complex, structured outputs. Concretely, Flow Matching learns by regressing onto conditional velocity fields that transport individual samples; integrating this learned dynamics then yields a model capable of generating the full data distribution~\citep{lipmanflow,chen2023flow2,tong2023improving,esser2024scaling,kornilov2024optimal,sriram2024flowllm,albergo2022building}.

However, this paradigm becomes significantly more challenging when we move to discrete combinatorial structures. Objects such as permutations, matchings, and tours obey strict structural constraints. For instance, a permutation matrix must contain exactly one “1” in each row and column. As a result, these objects reside on highly constrained geometric spaces rather than in the unconstrained ambient space $\mathbb{R}^{n\times n}$, making the direct application of standard flow-based methods nontrivial.

Most existing approaches deal with this difficulty by relaxing the problem. They replace permutations with continuous approximations (for example, using entropy-regularized Sinkhorn layers) and later round the solution back to a discrete one. While this makes optimization easier, it changes the geometry of the problem and can introduce bias or inconsistencies between training and testing. In particular, these methods usually produce a single softened solution and do not naturally represent uncertainty when multiple permutations are equally valid~\citep{sinkhorn1967diagonal}.
This is a real limitation in many settings where the supervision is set-valued rather than single-valued: there may be many feasible (and essentially equally good) permutations consistent with the input. For example, in matching or assignment problems with symmetries, repeated or indistinguishable items, or near-ties in cost, multiple permutations can explain the same observation. Yet a deterministic pipeline will still collapse everything into a single argmax permutation, committing to one choice and discarding alternative valid orderings. In contrast, what we often want is not one “best” permutation but a distribution over optimal permutations, allowing us to quantify ambiguity, sample diverse valid solutions, and propagate uncertainty downstream. Furthermore, entropy regularization introduces a well-known bias: the entropic term favors high-entropy  solutions, pushing mass toward the interior of the Birkhoff polytope and away from permutation vertices. This alters the optimization landscape relative to the original combinatorial problem and, when combined with discrete rounding at inference time, leads to a mismatch between training and inference objectives~\citep{janati2020debiased}.

In this work, we present Affine Permutation Flow (PermFlow)\footnote{https://github.com/yimengmin/PermFlow}. Instead of relaxing the structural constraints into an unconstrained continuous space, we  design flow dynamics that respect the geometry of permutation matrices from the outset. We use Conditional Flow Matching (CFM) to learn velocity fields. By projecting the learned velocities onto the correct tangent space (which enforces zero row and column sums), we ensure that the structural constraints are preserved throughout training. This allows us to train flow-based models without relying on entropy regularization or ad hoc projection steps.
Our approach generates a set of permutations by sampling different initial noise states and integrating the learned flow dynamics. As mentioned above, in many tasks, the correct ordering is inherently ambiguous, and multiple permutations may be equally consistent with the input. Unlike deterministic relaxations such as Sinkhorn, which collapse predictions to a single mode, our stochastic flow model captures this ambiguity by representing a distribution over permutations and sampling multiple valid combinatorial solutions~\citep{emami2018learning,mena2018learning}.
\section{Related Work}
\label{sec:related}
The dominant paradigm for differentiable permutation prediction relaxes
permutation matrices $\mathcal{P}$ to the Birkhoff polytope $\mathcal{B}_+$
via entropy-regularized optimal transport
\citep{sinkhorn1964relationship, cuturi2013sinkhorn}, with applications to
latent permutation learning \citep{mena2018learning}, sorting
\citep{grover2019stochastic}, and combinatorial RL
\citep{emami2018learning}. The entropic term biases solutions toward the
polytope interior \citep{janati2020debiased}, and deterministic Sinkhorn
cannot represent multimodal posteriors; Gumbel-Sinkhorn
\citep{mena2018learning} adds noise but empirically concentrates around a
single soft solution. Geometric alternatives include stick-breaking
parameterizations of $\mathcal{B}_+$ \citep{linderman2018reparameterizing}
and Riemannian Langevin sampling with retraction-based integrators
\citep{birdal2019probabilistic}. We share the geometric viewpoint but
work in the affine setting, which admits a closed-form orthogonal
projector and removes the need for retractions, and we learn
input-conditioned dynamics rather than fixed MCMC. Our construction
builds on flow matching and Conditional Flow Matching, which have been extended to optimal-transport
couplings \citep{kornilov2024optimal,lipmanflow,tong2023improving}, Riemannian manifolds
\citep{chen2023flow2}, and discrete state spaces
\citep{campbell2024generative, gat2024discrete}.
By contrast, generic manifold-ODE schemes enforce constraints only
approximately: under a Lipschitz assumption on the vector field,
the violation is bounded by a factor growing as $e^{Lt}$ in integration
time and $\mathcal{O}(h^p)$ in step size $h$, where $L$ is the Lipschitz
constant \citep{hairer1993solving}. This exponential amplification
carries over to learned vector fields, where projection or retraction
steps are needed to keep iterates close to the constraint set
\citep{lou2020neural, falorsi2020neural, mathieu2020riemannian,
chen2023flow2}, so feasibility degrades with both network expressivity
(which controls $L$) and integration horizon.

Our contributions are threefold:
1. We propose PermFlow, a CFM framework that flows on the affine subspace of matrices with unit row and column sums; a tangent-space projection keeps these constraints exact for any network, without entropic smoothing. 2. We show that pairing each noise sample with one valid permutation splits the noise space into regions, one per answer, whose sizes set the probabilities of the answers, so a single model can represent several valid permutations with chosen weights. 3. On visual sorting, jigsaw puzzles and ranking data, PermFlow recovers multiple valid permutations in close to the right proportions.

\section{Background}
\label{sec:background}
\paragraph{Sinkhorn Iteration for Soft Permutations}
Given a score matrix $S \in \mathbb{R}^{n \times n}$, the goal of permutation prediction is often formulated as a linear assignment problem
\begin{equation}
\max_{P \in \mathcal{P}_n} \langle S, P \rangle ,
\end{equation}
where $\mathcal{P}_n$ is the set of $n\times n$ permutation matrices. Because $\mathcal{P}_n$ is discrete, a common relaxation replaces it with the Birkhoff polytope
\begin{equation}
\mathcal{B}_+ =
\{X \in \mathbb{R}_+^{n\times n} \mid
X\mathbf{1}=\mathbf{1},\;
X^\top\mathbf{1}=\mathbf{1}\},
\end{equation}
the set of doubly stochastic matrices. To obtain a differentiable solution, entropy regularization is introduced:
\begin{equation}
\max_{X\in\mathcal{B}_+}
\langle S,X\rangle + \tau H(X),
\end{equation}
where $H(X)$ is the Shannon entropy and $\tau>0$ is a temperature parameter. The solution takes the form
\begin{equation}
X^\star = \mathrm{Diag}(u)\,K\,\mathrm{Diag}(v),
\qquad
K=\exp(S/\tau),
\end{equation}
where scaling vectors $u$ and $v$ enforce row and column normalization.  
These vectors can be computed efficiently by alternating row and column normalization~\citep{sinkhorn1964relationship}. In practice, the resulting matrix $X$ is a soft permutation (doubly stochastic matrix).  
As $\tau \to 0$, the solution becomes closer to a permutation matrix but gradients become unstable. Larger values of $\tau$ produce smoother but less discrete solutions.

Although Sinkhorn provides a useful differentiable relaxation, it typically produces a single softened solution. When multiple permutations are valid, for example due to symmetries or near-ties in scores, the entropy-regularized objective selects one solution in the Birkhoff polytope and collapses the others. As a result, Sinkhorn-based methods do not naturally represent uncertainty or generate diverse valid permutations.

\paragraph{Flow Matching and Conditional Flow Matching}

Flow Matching is a generative framework that learns a continuous transport process transforming samples from a source distribution $p_0$ into samples from a target distribution $p_1$.
Given noisy samples $\mathbf{x}_0 \sim p_0$ and clean samples $\mathbf{x}_1 \sim p_1$, the goal is to learn a time-dependent velocity field
$
u_t(\mathbf{x}_t)
$ that transports probability mass from $p_0$ to $p_1$. Let $\psi_t(\mathbf{x}_0 \mid \mathbf{x}_1)$ denote an interpolant that defines the trajectory between $\mathbf{x}_0$ and $\mathbf{x}_1$. The interpolant satisfies the boundary conditions
\begin{equation}
\psi_0(\mathbf{x}_0 \mid \mathbf{x}_1) = \mathbf{x}_0,
\qquad
\psi_1(\mathbf{x}_0 \mid \mathbf{x}_1) = \mathbf{x}_1 .
\end{equation}

At time $t \in [0,1]$, the intermediate point is $
\mathbf{x}_t = \psi_t(\mathbf{x}_0 \mid \mathbf{x}_1).$ The conditional velocity field is defined as the time derivative of the interpolant:
\begin{equation}
u_t(\mathbf{x}_t \mid \mathbf{x}_1)
=
\frac{d}{dt}\psi_t(\mathbf{x}_0 \mid \mathbf{x}_1).
\end{equation}

For a velocity field $u_t$ to generate the evolving distribution $p_t$, it must satisfy the continuity equation
\begin{equation}
\frac{\partial}{\partial t} p_t(\mathbf{x})
=
-\nabla \cdot \big(p_t(\mathbf{x}) u_t(\mathbf{x})\big),
\end{equation}
where $\nabla \cdot$ denotes the divergence operator. Flow matching objective trains a parameterized velocity model $u_\theta(\mathbf{x}_t,t)$ to approximate the true marginal velocity field:
\begin{equation}
\mathcal{L}_\text{FM}
=
\mathbb{E}_{t,\mathbf{x}_t}
\left[
\left\|
u_\theta(\mathbf{x}_t,t) - u_t(\mathbf{x}_t)
\right\|^2
\right].
\end{equation}

However, computing the marginal velocity $u_t(\mathbf{x}_t)$ requires integrating over all trajectories connecting $p_0$ and $p_1$, which is generally intractable. To address this issue, CFM conditions the velocity field on individual data samples $\mathbf{x}_1$. The resulting CFM objective becomes
\begin{equation}
\mathcal{L}_\text{CFM}
=
\mathbb{E}_{t,\mathbf{x}_0,\mathbf{x}_1}
\left[
\left\|
u_\theta(\mathbf{x}_t,t) -
u_t(\mathbf{x}_t \mid \mathbf{x}_1)
\right\|^2
\right].
\end{equation}

Importantly, the conditional and unconditional objectives share the same gradient:
\begin{equation}
\nabla_\theta \mathcal{L}_\text{FM}
=
\nabla_\theta \mathcal{L}_\text{CFM},
\end{equation}

making the conditional formulation both tractable and theoretically consistent~\citep{lipmanflow}.

\section{Geometry-Preserving Permutation Flow}
\label{sec:permflow}
Before describing the construction, we highlight the kind of guarantee
we aim for. Existing approaches keep iterates feasible only
approximately: Sinkhorn alternates row and column normalization
and leaves a residual that shrinks with iteration count, while generic
manifold-ODE schemes bound the constraint violation only up to a factor
that grows as $e^{Lt}$ under a Lipschitz assumption on the vector field, with $L$ denotes the Lipschitz constant.
In both cases feasibility is approximate: Sinkhorn typically leaves a
residual after finitely many normalisation steps, which vanishes as the
iteration converges, and
the manifold-ODE guarantee is an upper bound that can grow with network
expressivity and integration time.
\begin{wrapfigure}{r}{0.5\textwidth}
\vspace{-10pt}
\centering
\resizebox{\linewidth}{!}{%
\begin{tikzpicture}[
    >=Stealth,
    font=\small,
    axislabel/.style={font=\footnotesize},
    methodlabel/.style={font=\footnotesize\itshape},
    scale=0.72, transform shape
]
\draw[->, thick] (0,0) -- (8.5,0) node[right] {iteration / time $t$};
\draw[->, thick] (0,0) -- (0,4) node[above right, font=\large, xshift=-2pt] {row/column-sum violation};

\draw[thick, blue!55!black, smooth, samples=80, domain=0.05:8, opacity=0.7]
    plot (\x, {2.8*exp(-0.55*\x)});
\node[blue!55!black, methodlabel, anchor=west, font=\large, opacity=0.85]
    at (3.4, 0.95) {Sinkhorn: residual $\to 0$ as the iteration converges};

\draw[thick, dashed, red!65!black, smooth, samples=80, domain=0:4.2, opacity=0.7]
    plot (\x, {0.08*exp(0.85*\x)});
\node[red!65!black, methodlabel, anchor=west, opacity=0.85, font=\large]
    at (3.6, 3.2) {Manifold ODE: error \emph{bound} $\propto e^{Lt}$};

\draw[line width=6pt, green!60!black, opacity=0.18]
    (0,0) -- (8.2,0);
\draw[line width=2.2pt, green!45!black]
    (0,0) -- (8.2,0);
\foreach \x in {1, 2, 3, 4, 5, 6, 7, 8} {
    \fill[green!40!black] (\x, 0) circle (2pt);
}

\node[green!35!black, font=\bfseries, anchor=west]
    at (2.5, -0.55) {PermFlow (ours): $0$ at every Euler step};
\draw[->, green!40!black, thick] (5.2, -0.4) -- (5.2, -0.05);

\end{tikzpicture}}
\caption{Row/column-sum violation. With finitely many normalisation steps,
Sinkhorn typically leaves a residual, which vanishes as the iteration
converges; for integrators on
general manifolds without projection, standard error bounds contain a
factor $e^{Lt}$ (dashed: a bound, not necessarily the realised violation).
With the projected velocity, every Euler step of \textbf{PermFlow}
preserves the linear constraints.}
\label{fig:feasibility-comparison}
\vspace{-10pt}
\end{wrapfigure}
We aim for a stronger property. We work on the affine subspace of
matrices with row and column sums equal to one, and build a projector
$\mathcal{C}$ such that $\mathcal{C}(U)$ lies in the corresponding
tangent space for every matrix $U$, as an algebraic identity.
Substituted into the flow, the row- and column-sum constraints hold exactly at every $t$ at which the solution exists, for any $f_\theta$, with no Lipschitz or smoothness assumption beyond ODE existence. Figure~\ref{fig:feasibility-comparison} compares this guarantee against existing strategies.

We model permutations through continuous dynamics on the affine constraint
manifold. Let $\1 \in \R^n$ denote the all-ones vector, and endow
$\R^{n\times n}$ with the Frobenius inner product
$\inner{A}{B} = \tr(A^\top B)$. The affine constraint manifold is the set
of matrices whose rows and columns each sum to one,
\begin{equation}
    \B_a \;=\; \{X \in \R^{n\times n} : X\1 = \1,\ X^\top \1 = \1\},
\end{equation}
and at any point of $\B_a$ the admissible directions of motion are those
that preserve these row and column sums:
\begin{equation}
    T \;=\; \{\Delta \in \R^{n\times n} : \Delta\1 = 0,\ \Delta^\top\1 = 0\}.
\end{equation}
Note that $T$ does not depend on the base point, since $\B_a$ is affine.

\begin{algorithm}[t]
\caption{PermFlow: conditional flow matching on $\B_a$}
\label{alg:general-perm-cfm}
\small
\begin{algorithmic}[1]
\Require
  contexts $\mathbf{c}$ with a target set
  $\mathcal{T}(\mathbf{c})=\{P^1,\ldots,P^M\}$ and coupling weights $w$
  (set-valued; $w\equiv0$ is the nearest-target rule) or one observed
  permutation $P$ (sample-valued); network $f_\theta$; $\sigma_0$; steps $S$; samples $K$.
\State $\C(U) \gets (I-J)\,U\,(I-J)$, \ $J=\tfrac1n\1\1^\top$;\quad $X_0 \gets J$
\State \textbf{Training.}
\For{each minibatch}
    \State $\varepsilon \gets \C(G)/\|\C(G)\|_F$, \ $G\sim\mathcal{N}(0,I)$;\quad
           $\tilde{X}_0 \gets X_0 + \sigma_0\,\varepsilon$
        \hfill\Comment{$\tilde X_0\in\B_a$}
    \If{set-valued}
        \State $P^\star \gets P^{m}$, \ $m=\argmax_{m} \langle \tilde{X}_0, P^m\rangle_F + w_m$
        \hfill\Comment{coupling cell, fresh every step}
    \Else
        \State $P^\star \gets P$
        \hfill\Comment{independent coupling}
    \EndIf
    \State $t \sim \mathcal{U}(0,1)$;\quad
           $X_t \gets (1-t)\,\tilde{X}_0 + t\,P^\star$
    \State $\mathcal{L} \gets \|\C(f_\theta(X_t, t, \mathbf{c})) - (P^\star - \tilde X_0)\|_F^2$;\quad
           update $\theta$
\EndFor
\vspace{0.3em}
\State \textbf{Inference.}
\For{$k = 1, \ldots, K$}
    \State $\varepsilon^{(k)} \gets \C(G)/\|\C(G)\|_F$, \ $G\sim\mathcal{N}(0,I)$;\quad
           $X^{(k)} \gets X_0 + \sigma_0\,\varepsilon^{(k)}$
        \hfill\Comment{as in line 4}
    \For{$s = 0, \ldots, S-1$}
        \State $X^{(k)} \gets X^{(k)} + \tfrac{1}{S}\, \C(f_\theta(X^{(k)}, s/S, \mathbf{c}))$
            \hfill\Comment{stays in $\B_a$}
    \EndFor
    \State $\hat{P}^{(k)} \gets \mathrm{Round}(X^{(k)})$
        \hfill\Comment{Hungarian, on the real matrix}
\EndFor
\State \Return $\{\hat{P}^{(k)}\}_{k=1}^K$
\end{algorithmic}
\end{algorithm}

To map arbitrary matrices onto $T$, we use the centering operator
$\C : \R^{n\times n} \to \R^{n\times n}$,
\begin{equation}
    \C(U)
    \;=\;
    U
    \;-\; \tfrac{1}{n} U \1\1^\top
    \;-\; \tfrac{1}{n} \1\1^\top U
    \;+\; \tfrac{1}{n^2} \1\1^\top U \1\1^\top.
\end{equation}
The first correction term subtracts each row's mean, the second subtracts
each column's mean, and the third adds back the global mean to avoid
double-counting. A direct calculation confirms that $\C(U)\1 = 0$ and
$\1^\top \C(U) = 0$, so $\C(U)$ always lies in $T$.

\begin{theorem}[Geometric correctness of PermFlow]
\label{thm:main}
The centering operator $\C$ is the orthogonal projector onto $T$.
Consequently, the ODE
\begin{equation}
    \dot{X}_t \;=\; \C\bigl(f_\theta(X_t, t)\bigr), \qquad X_0 \in \B_a,
\end{equation}
satisfies $X_t \in \B_a$ for every $t$ in the interval of existence of its
solution, for any choice of $f_\theta$.
\end{theorem}
A full proof of Theorem~\ref{thm:main} is given in Appendix~\ref{app:proof}. Note that $\B_a$ is not the Birkhoff polytope $\B_+$, but its \textbf{affine hull}:
only the equality constraints are imposed, nonnegativity is relaxed, and
$X_t$ is therefore not doubly stochastic and may have negative entries. This costs
nothing at decoding time. Double stochasticity is not required in order to
recover a valid permutation, because Hungarian rounding solves a linear
assignment problem over an arbitrary real matrix, and the linear
assignment optimum is unchanged by adding a constant to any row or column
and by a global positive rescaling, so its correctness requires neither
nonnegativity nor unit row and column sums. We refer to Appendix~\ref{app:rounding} for
details.

Three features make the argument particularly
clean. First, the four-term formula for $\C$ can be
rewritten as a single sandwich,
$\C(U) = (I - J)\,U\,(I - J)$, where $J = \tfrac{1}{n}\1\1^\top$ is the
rank-one matrix that averages a vector. Since $I - J$ is itself an orthogonal
projector, every property we need follows from the single identity
$(I - J)^2 = I - J$: idempotency, self-adjointness, and the fact that $\C$
maps onto exactly $T$. Second, the construction is not specific to
permutation matrices. Any affine subspace of $\R^{n\times n}$ defined by
linear row or column constraints admits an analogous projector built the
same way, so the trick is an instance of a general geometric principle.
Third, the flow-invariance claim requires no Gr\"onwall inequality, no
Lipschitz assumption on $f_\theta$, and no smoothness beyond what is needed
for the ODE to exist. The constraint is preserved as an exact algebraic
identity along trajectories, not as a small-error bound that grows with
time. This is why PermFlow can claim that the row and column sums are
preserved exactly rather than approximately, a strictly stronger
guarantee than entropy-regularized relaxations can offer.

Putting the pieces together, we parameterize a velocity field with a neural
network $f_\theta(X,t) \in \R^{n\times n}$ and project its output onto the
tangent space before using it,
\begin{equation}
    v_\theta(X,t) \;=\; \C\bigl(f_\theta(X,t)\bigr).
\end{equation}
The permutation dynamics are then given by the neural ODE
\begin{equation}
    \frac{dX_t}{dt} \;=\; v_\theta(X_t, t), \qquad X_0 \in \B_a.
\end{equation}
Because the velocity always lies in $T$, the row and column sums of $X_t$ do
not change along a trajectory, so $X_t$ stays on $\B_a$ for all $t$ no matter
what the network outputs.

\begin{figure}[htbp]
\centering
\begin{tikzpicture}[node distance=4mm,scale=0.8]

\begin{scope}[yshift=1.0cm,scale=0.6]
  \node[font=\small, black!60] at (0, 1.5) {$X_0 = \tfrac{1}{n}\mathbf{1}\mathbf{1}^\top$};
  \foreach \col in {0,1,2}{
    \foreach \row in {0,1,2}{
      \fill[gray!25, draw=black!40]
        ({-0.62+\col*0.42},{0.22-\row*0.42}) rectangle ++(0.40,0.40);
    }
  }
  \draw[black, line width=1.2pt, rounded corners=2pt]
    (-0.64,-0.64) rectangle (0.64,0.64);

  \node[font=\Large, black!70] at (1.7, 0) {$+$};

  \begin{scope}[xshift=3.1cm]
    \node[font=\small, black!60] at (0, 1.5) {$\sigma_0\,\varepsilon$};
    \foreach \x/\y/\s in {
      -0.38/ 0.30/2.5,  0.10/ 0.42/1.8,  0.35/ 0.10/2.8,
       0.22/-0.35/2.2, -0.20/-0.28/2.0, -0.40/-0.05/1.5,
       0.00/ 0.00/1.2,  0.42/ 0.38/1.9, -0.10/ 0.18/1.3,
       0.28/-0.10/1.6}{
      \fill[black!55] (\x,\y) circle (\s pt);
    }
    \draw[black!60, thick, rounded corners=2pt, dashed]
      (-0.55,-0.55) rectangle (0.55,0.55);
  \end{scope}

  \node[font=\small, black!70] at (5.3, 1.5) {$\mathcal{C}(\cdot)$};
  \draw[arr, black!80, line width=2pt] (4.0,0) -- (6.2,0)
    node[midway, below, font=\scriptsize, black!60]{tangent proj};

  \begin{scope}[xshift=7.2cm]
    \node[font=\small, black] at (0, 1.5) {$\tilde{X}_0$};
    \foreach \col/\row/\op in {
      0/0/60, 1/0/25,  2/0/40,
      0/1/30, 1/1/50,  2/1/20,
      0/2/45, 1/2/22,  2/2/55}{
      \fill[gray!\op!white, draw=black!40]
        ({-0.62+\col*0.42},{0.22-\row*0.42}) rectangle ++(0.40,0.40);
    }
    \draw[black, line width=1.2pt, rounded corners=2pt]
      (-0.64,-0.64) rectangle (0.64,0.64);
  \end{scope}
\end{scope}

\begin{scope}[xshift=0cm,yshift=-1cm, yshift=0cm,scale=0.7]
  \foreach \col/\row/\op in {0/0/50, 1/0/25, 0/1/30, 1/1/45}{
    \fill[gray!\op!white, draw=black!40, thin]
      ({-0.23+\col*0.23},{0.01-\row*0.23}) rectangle ++(0.21,0.21);
  }
  \draw[black!70, rounded corners=2pt, line width=0.9pt]
    (-0.25,-0.25) rectangle (0.25,0.25);
  \foreach \x/\y in {-0.31/0.12, 0.31/0.18, 0.28/-0.28,
                      -0.26/-0.26, 0.34/0.0, -0.08/0.32}{
    \fill[black!55] (\x,\y) circle (1.6pt);
  }
  \draw[black!50, dashed, rounded corners=2pt, line width=0.8pt]
    (-0.38,-0.38) rectangle (0.38,0.38);

  \node[coordinate] (X0) at (0,0) {};
  \node[font=\scriptsize, black, below=10pt] at (0,0) {\large $\tilde{X}_0$};

  \node[pbox] (Pa) at (6.2, 1.2) {\large $P^a$};
  \node[pbox] (Pb) at (6.2,-1.2) {\large $P^b$};

  \draw[arr, black, bend left=18]      (X0) to (Pa);
  \draw[arr, black!70, bend right=18,
        dash pattern=on 4pt off 2pt]   (X0) to (Pb);

  \node[circle, fill=black!60, inner sep=3.5pt] (Xta) at (2.1, 0.5) {};
  \node[circle, fill=black!30, inner sep=3.5pt] (Xtb) at (2.1,-0.5) {};

  \draw[black!60, thick] (0.38,0) -- (Xta.west);
  \draw[black!60, thick] (Xta.east) -- (Pa.west);
  \draw[black!35, thick, dashed] (0.38,0) -- (Xtb.west);
  \draw[black!35, thick, dashed] (Xtb.east) -- (Pb.west);


  \node[vbox] (vnet) at (2.1,-2.2) {\large $v_\theta$};
  \draw[darr, black!70] (Xta) -- (vnet);
  \draw[darr, black!40] (Xtb) -- (vnet);

  \node[circle, draw=black, fill=gray!20, thick,
        inner sep=2.5pt, font=\scriptsize] (loss) at (2.1,-3.5) {$\mathcal{L}$};
  \draw[darr, black] (vnet) -- (loss);
\end{scope}

\begin{scope}[xshift=7.2cm, yshift=-1cm,scale=1.2]

  \begin{scope}[yshift=1.50cm]
    \noisygridicon{0/0/50, 1/0/25, 0/1/35, 1/1/55}%
                  {-0.24/0.22, 0.24/0.22, 0.24/-0.22, -0.24/-0.22, 0.00/0.25}
    \node[coordinate] (Xk1) at (0,0) {};
  \end{scope}

  \begin{scope}[yshift=0.75cm]
    \noisygridicon{0/0/60, 1/0/20, 0/1/40, 1/1/30}%
                  {0.25/0.18, -0.25/0.22, 0.22/-0.24, -0.23/-0.18, 0.00/-0.25}
    \node[coordinate] (Xk2) at (0,0) {};
  \end{scope}

  \begin{scope}[yshift=0.00cm]
    \noisygridicon{0/0/30, 1/0/55, 0/1/22, 1/1/60}%
                  {-0.25/-0.18, 0.14/0.25, 0.25/-0.10, -0.10/0.24, 0.23/0.12}
    \node[coordinate] (Xk3) at (0,0) {};
  \end{scope}

  \begin{scope}[yshift=-0.75cm]
    \noisygridicon{0/0/45, 1/0/55, 0/1/28, 1/1/40}%
                  {0.22/0.24, -0.25/0.12, 0.18/-0.24, 0.25/0.05, -0.10/-0.25}
    \node[coordinate] (Xk4) at (0,0) {};
  \end{scope}

  \begin{scope}[yshift=-1.50cm]
    \noisygridicon{0/0/55, 1/0/30, 0/1/48, 1/1/20}%
                  {-0.18/0.25, 0.25/0.18, -0.25/-0.14, 0.14/-0.25, 0.24/0.0}
    \node[coordinate] (Xk5) at (0,0) {};
  \end{scope}

  \begin{pgfonlayer}{background}

  \draw[flowbox]
    (-0.35, 1.80) rectangle (3.00,-1.80);
  \node[font=\large, black!55, fill=gray!5,
        inner sep=2pt]
    at (1.5, 2.08) {ODE integration };
   \end{pgfonlayer}
  \node[circle, fill=black!65, inner sep=2.0pt] (w1) at (2.2, 1.0) {};
  \node[circle, fill=black!55, inner sep=2.0pt] (w2) at (2.2, 0.5) {};
  \node[circle, fill=black!45, inner sep=2.0pt] (w3) at (2.2, 0.0) {};
  \node[circle, fill=black!30, inner sep=2.0pt] (w4) at (2.2,-0.5) {};
  \node[circle, fill=black!20, inner sep=2.0pt] (w5) at (2.2,-1.0) {};

  \draw[arr, black!70, bend left=5]  (Xk1) -- (w1);
  \draw[arr, black!60]               (Xk2) -- (w2);
  \draw[arr, black!50, bend right=5] (Xk3) -- (w3);
  \draw[arr, black!35, dashed,
        -{Stealth[length=6pt,width=4pt]},
        bend right=5]                (Xk4) -- (w4);
  \draw[arr, black!25, dashed,
        -{Stealth[length=6pt,width=4pt]},
        bend right=8]                (Xk5) -- (w5);


  \node[vbox, minimum width=1.5cm] (hung) at (4.0, 0) {$\mathrm{Round}$};

  \draw[arr, black!65, bend left=10]  (w1) -- (hung);
  \draw[arr, black!55, bend left=5]   (w2) -- (hung);
  \draw[arr, black!45]                (w3) -- (hung);
  \draw[arr, black!30, dashed,
        -{Stealth[length=6pt,width=4pt]},
        bend right=5]                 (w4) -- (hung);
  \draw[arr, black!20, dashed,
        -{Stealth[length=6pt,width=4pt]},
        bend right=10]                (w5) -- (hung);

  \foreach \k/\ys in {1/1.50, 2/0.75, 3/0.00, 4/-0.75, 5/-1.50}{
    \node[outbox] (out\k) at (5.8,\ys) {\large $\hat{P}^{(\k)}$};
  }
  \foreach \k in {1,2,3}{
    \draw[arr, black!60] (hung.east) -- (out\k.west);
  }
  \foreach \k in {4,5}{
    \draw[black!40, dashed,
          -{Stealth[length=6pt,width=4pt]}, line width=1.5pt]
      (hung.east) -- (out\k.west);
  }


\end{scope}

\draw[thick, black!25, dashed] (6.2, 1.8) -- (6.2,-4.0);

\end{tikzpicture}
\caption{Our model:
noise is added to the uniform $X_0$ and projected to the  tangent space.
Training learns a velocity field via conditional flow matching to the nearest valid permutation, while inference integrates the ODE and rounds to obtain candidate permutations.
}
\label{fig:1}
\end{figure}

\paragraph{Coupling, training and inference}
\label{sec:coupling}
As shown in Figure~\ref{fig:1}, our model learns a generative flow directly in the space of permutation matrices; Algorithm~\ref{alg:general-perm-cfm} summarises training and inference.
We start from a uniform doubly stochastic matrix
$
X_0 = \tfrac{1}{n}\mathbf{1}\mathbf{1}^\top
$
and add Gaussian noise to obtain a perturbed initialization
$
\tilde{X}_0 = X_0 + \sigma_0\,\mathcal{C}(\varepsilon)/\|\mathcal{C}(\varepsilon)\|_F,
$ where $\sigma_0$ is noise scale and  $\varepsilon  \sim \mathcal{N}(0, I)$.
To preserve the structural constraints of permutation matrices, the noise is projected onto the tangent space using $
\mathcal{C}(U)
$,
which enforces zero row and column sums, and normalised so that
$\|\tilde X_0-X_0\|_F=\sigma_0$ independently of $n$. Given a noisy matrix $\tilde{X}_0$, we define a coupling with valid permutation targets using a
nearest-valid-target-wins rule.
Among the set of valid permutations $\{P^1,\dots,P^M\}$,
the target permutation $P^\star$ is selected as
\begin{equation}
\label{eq:coupling}
P^\star = \argmin_{P \in \{P^1,\dots,P^M\}}
\|\tilde{X}_0 - P\|_F .
\end{equation}
When the data also specify target weights $\pi_m$, the rule is shifted to
$P^\star=\argmax_m \langle\tilde X_0,P^m\rangle_F+w_m$, with $w$ chosen so
that the cell of each target has mass $\pi_m$; for two targets this is a
closed-form threshold on $\langle\varepsilon,P^a-P^b\rangle_F$, which we call
the \emph{calibrated} coupling, derived in Appendix~\ref{app:unbiased}. When each
training input comes with a single observed permutation instead of a set of
targets, $P^\star$ is that permutation, independently of $\tilde X_0$, which is the
independent coupling of standard CFM; Appendix~\ref{app:unbiased} discusses
what the theory does and does not cover for our spherical source.

This coupling defines the endpoint of the transport path used during training. We then sample interpolation points
\begin{equation}
\label{eq:path}
X_t = (1-t)\tilde{X}_0 + t P^\star , \quad t \sim \mathcal{U}(0,1),
\end{equation}
and train a neural velocity field $v_\theta(X_t,t)$ using conditional flow matching to predict the transport dynamics that move samples toward the target permutation distribution.
The objective minimizes the discrepancy between the predicted velocity and the analytically defined conditional velocity along this path:
\begin{equation}
\label{eq:cfm}
    \mathcal{L}_{\text{CFM}} = \mathbb{E}_{t, \tilde{X}_0, P^\star}
    \left[ \left\| v_\theta(X_t, t, h) - \dot{X} \right\|_F^2 \right],
\end{equation}
where $h = g_\phi(c)$ is the context encoding and $\dot{X} = P^\star - \tilde{X}_0$ is the
analytically defined constant conditional velocity along the linear interpolation path.
The nearest-target coupling plays a central role: by anchoring each noisy initialization
to its closest valid permutation, it ensures that trajectories originating from different
regions of the $\B_a$ are directed toward distinct permutations.
This prevents gradient cancellation, which would otherwise occur when conflicting
targets impose opposing velocity directions on nearby points and encourages the  learned flow to develop bifurcating dynamics. As a result, the model captures the  full multimodal structure of the target distribution, assigning distinct trajectory branches to each valid permutation mode rather than collapsing to a single averaged solution.

\label{sec:inference}
At inference, $K$ independent noisy initializations: $ \tilde{X}_0^{(k)} = \frac{1}{n}\mathbf{1}\mathbf{1}^\top + \sigma_0\,\mathcal{C}(\varepsilon^{(k)})/\|\mathcal{C}(\varepsilon^{(k)})\|_F$, $ \varepsilon^{(k)} \sim \mathcal{N}(0, I)$, are drawn, the noise being projected onto the tangent space by $\mathcal{C}$ so that $\tilde{X}_0^{(k)} \in \B_a$. Each is evolved independently through the learned ODE dynamics via $S$ Euler steps: $X^{(k)} \leftarrow X^{(k)} + \frac{1}{S} \cdot v_\theta\!\left(X^{(k)}, s/S, h\right)$ for $s=0,\dots,S-1$.
Because each trajectory starts from a distinct noise seed, the learned velocity field routes them toward different modes of the permutation distribution, producing diverse continuous endpoints on the affine constraint manifold.
Each endpoint is then mapped to a discrete permutation via Hungarian rounding: $\hat{P}^{(k)} = \mathrm{Round}\!\left(X^{(k)}\right)$,
yielding a set of candidate permutations $\{\hat{P}^{(k)}\}_{k=1}^{K}$ that collectively cover the valid modes of the target distribution.

\paragraph{Difference from Entropy-regularized Sinkhorn}
Unlike Sinkhorn, PermFlow uses no entropy
regularization. Here, “unbiased” refers to avoiding the bias this regularization introduces, not to an
unbiased estimator or a guarantee of exact distributional sampling. We refer to Appendix~\ref{app:unbiased} for the
mode probabilities induced by the set-valued coupling.
\section{Experiments}
\label{sec:experiments}
We evaluate on three tasks that differ in the source of multiplicity: an
ambiguous input with two weighted answers (\S\ref{sec:mnist}), a symmetry
among identical items (\S\ref{sec:jigsaw}), and a heterogeneous population
observed through samples (\S\ref{sec:mallows}). Within a task all methods
share the encoder, transformer, optimiser, schedule and data; we report the
final-epoch model, as means $\pm$ standard deviations over three seeds
unless noted. In PermFlow the velocity network is a transformer over one token per
item, to which the current state $X_t$ (flattened, linearly projected, and
added to an embedding of $t$) is appended as one extra \emph{state token};
the network uses $\sigma_0=1$ and $S=10$ Euler steps unless stated, and
permutations are decoded by Hungarian rounding.

\subsection{Visual Sorting with Two Weighted Modes}
\label{sec:mnist}

Each input is a sequence of $N=9$ MNIST digits to be sorted by value. In
half of the sequences one image is a blend
$\alpha\,\mathbf{x}_a + (1{-}\alpha)\,\mathbf{x}_b$ of two digits, with
$\alpha \sim \mathrm{Beta}(2,2)$ clipped to $[0.2, 0.8]$, so the sequence
has two valid orders $P_a$ and $P_b$; we define the target distribution to
place mass $\alpha$ on $P_a$ and $1-\alpha$ on $P_b$. A good sampler should
therefore not only produce both orders, but produce them in the right
proportions. We report clean
accuracy, Coverage@$K$ (both orders among $K$ samples) and the calibration
error $\mathbb{E}|\hat f_a-\alpha|$ of the sampled frequency of $P_a$, and
compare PermFlow (nearest and calibrated couplings, and its
state-independent variant) with Gumbel-Sinkhorn (GS) trained on the same
backbone, where each ambiguous example is supervised with one valid order:
always $P_a$, always $P_b$, either with probability $\tfrac12$, or $P_a$ with
probability $\alpha$, and decoded by rounding the Sinkhorn output.

\begin{table}[t]
\centering
\small
\caption{Visual sorting with two weighted modes. All methods sort clean
sequences equally well, but only PermFlow with the calibrated coupling
samples the two valid orders in proportion to their weights. ``$\alpha$'':
training uses the mode weight.}
\label{tab:mnist_main}
\setlength{\tabcolsep}{4pt}
\begin{tabular}{lcccccc}
\toprule
 & & Clean acc.\ (\%) & \multicolumn{2}{c}{Coverage (\%) $\uparrow$} & \multicolumn{2}{c}{Calibration error $\downarrow$} \\
\cmidrule(lr){4-5}\cmidrule(lr){6-7}
Method & $\alpha$ & $K{=}1$ & $K{=}10$ & $K{=}100$ & $K{=}10$ & $K{=}100$ \\
\midrule
GS, only $P_a$ &  & 98.6{\scriptsize$\pm$0.1} & 0.0{\scriptsize$\pm$0.0} & 0.0{\scriptsize$\pm$0.0} & 0.495{\scriptsize$\pm$0.001} & 0.495{\scriptsize$\pm$0.000} \\
GS, only $P_b$ &  & 98.5{\scriptsize$\pm$0.4} & 0.0{\scriptsize$\pm$0.0} & 0.0{\scriptsize$\pm$0.0} & 0.504{\scriptsize$\pm$0.000} & 0.504{\scriptsize$\pm$0.000} \\
GS, random $P_a$ or $P_b$ &  & 98.6{\scriptsize$\pm$0.1} & 99.5{\scriptsize$\pm$0.1} & 99.8{\scriptsize$\pm$0.1} & 0.208{\scriptsize$\pm$0.001} & 0.177{\scriptsize$\pm$0.002} \\
GS, $P_a$ w.p.\ $\alpha$ & \checkmark & 98.6{\scriptsize$\pm$0.2} & 48.6{\scriptsize$\pm$1.4} & 69.0{\scriptsize$\pm$2.0} & 0.212{\scriptsize$\pm$0.003} & 0.196{\scriptsize$\pm$0.002} \\
\midrule
PermFlow, state-independent &  & 98.3{\scriptsize$\pm$0.2} & 98.6{\scriptsize$\pm$0.2} & 99.4{\scriptsize$\pm$0.2} & 0.213{\scriptsize$\pm$0.003} & 0.180{\scriptsize$\pm$0.002} \\
PermFlow, nearest coupling &  & 97.8{\scriptsize$\pm$0.3} & 98.8{\scriptsize$\pm$0.2} & 99.5{\scriptsize$\pm$0.2} & 0.212{\scriptsize$\pm$0.002} & 0.177{\scriptsize$\pm$0.002} \\
PermFlow, calibrated coupling & \checkmark & 97.9{\scriptsize$\pm$0.1} & 94.6{\scriptsize$\pm$0.2} & 99.4{\scriptsize$\pm$0.1} & \textbf{0.129{\scriptsize$\pm$0.002}} & \textbf{0.063{\scriptsize$\pm$0.000}} \\
\bottomrule
\end{tabular}
\end{table}

As shown in Table~\ref{tab:mnist_main}, all methods are equally good
point predictors on clean sequences; they separate on how they apportion mass
between the two valid orders. GS trained on a single fixed order (always
$P_a$ or always $P_b$) outputs only that order, with no coverage and
calibration error $0.50$. Supervised with either order at random, GS covers
both ($99.8\%$ at $K=100$) but splits them evenly, with calibration error
$0.177$, exactly like PermFlow with the nearest coupling ($0.177$) and its
state-independent variant ($0.180$). This is the value of a $50{:}50$
split: with two targets, the nearest-target rule divides the noise into two
cells of equal mass, whatever the value of $\alpha$. Giving $\alpha$ to GS through sampled
one-hot targets does not calibrate it: GS trained on $P_a$ with probability
$\alpha$ stays at $0.196$ and covers both orders on only $69\%$ of
instances. With the same information, PermFlow with the calibrated coupling
reaches $0.063$, close to the finite-sample floor of $0.037$: in GS the mode
weight must be recovered through the argmax of one perturbed score matrix,
which the loss does not calibrate, whereas in PermFlow it is written
directly into the mass of each noise cell. We refer to
Appendix~\ref{app:mnist} for an example input, the data, metrics, models and
reference values of the calibration error, and to
Appendix~\ref{app:unbiased} for the analysis of the cell masses.

\subsection{Jigsaw Puzzles: Unique Labels and Orbits of Identical Patches}
\label{sec:jigsaw}
\begin{table}[htbp]
\centering
\small
\caption{Jigsaw puzzles in the published setting ($\sigma=0.01$). Decoding a
single sample, PermFlow matches or improves on every published result.
Bold: best in row.}
\label{tab:jigsaw}
\setlength{\tabcolsep}{4pt}
\begin{tabular}{llcccccc}
\toprule
Metric & Grid & GS & DiffSort & EF-DiffSort & SymmDiff & PermFlow & PermFlow, $S{=}1$ \\
\midrule
\multirow{3}{*}{Kendall-$\tau$ $\uparrow$} & $4{\times}4$ & 0.3578 & 0.0374 & 0.0100 & 0.4859 & \textbf{0.5434{\scriptsize$\pm$0.0018}} & 0.5473 \\
 & $5{\times}5$ & 0.2430 & 0.0176 & 0.0034 & 0.2853 & \textbf{0.4383{\scriptsize$\pm$0.0009}} & 0.4421 \\
 & $6{\times}6$ & 0.1755 & 0.0095 & -0.0021 & 0.1208 & \textbf{0.3225{\scriptsize$\pm$0.0042}} & 0.3380 \\
\cmidrule(lr){1-8}
\multirow{3}{*}{Accuracy (\%) $\uparrow$} & $4{\times}4$ & 0.86 & 0.00 & 0.00 & 1.38 & \textbf{3.00{\scriptsize$\pm$0.06}} & 3.04 \\
 & $5{\times}5$ & 0.00 & 0.00 & 0.00 & 0.00 & \textbf{0.01{\scriptsize$\pm$0.01}} & 0.01 \\
 & $6{\times}6$ & \textbf{0.00} & \textbf{0.00} & \textbf{0.00} & \textbf{0.00} & \textbf{0.00{\scriptsize$\pm$0.00}} & 0.00 \\
\cmidrule(lr){1-8}
\multirow{3}{*}{Correct (\%) $\uparrow$} & $4{\times}4$ & 49.51 & 10.77 & 7.40 & 58.51 & \textbf{63.47{\scriptsize$\pm$0.03}} & 63.78 \\
 & $5{\times}5$ & 26.94 & 6.39 & 4.39 & 37.91 & \textbf{53.23{\scriptsize$\pm$0.05}} & 53.71 \\
 & $6{\times}6$ & 14.91 & 3.77 & 2.50 & 18.54 & \textbf{41.07{\scriptsize$\pm$0.44}} & 42.81 \\
\bottomrule
\end{tabular}
\end{table}

On the jigsaw benchmark of \citet{zhang2025symmetricdiffusers}, an MNIST
image is cut into a $g\times g$ grid ($n=16,25,36$) and the shuffled
patches must be reassembled. The published benchmark adds pixel noise
($\sigma=0.01$) so that every patch is unique and the answer is a single
permutation. Without it ($\sigma=0$), the blank background of a digit
yields many identical patches that can be exchanged without changing the
image, and the set of perfect reassemblies becomes an orbit $\mathcal{O}$
with a median of $5{,}040$, $4.8\times10^8$ and $6.4\times10^{15}$ members
at the three grids. PermFlow takes the whole orbit as its target set and
couples each noise sample to the nearest member of the orbit, so that
different noise samples are routed to different, equally correct
reassemblies.

\begin{table}[t]
\centering
\small
\caption{Jigsaw puzzles without pixel noise ($\sigma=0$). PermFlow reassembles
most images exactly and spreads its samples over many distinct perfect
reassemblies. Ablations are indented.}
\label{tab:orbit}
\setlength{\tabcolsep}{4pt}
\begin{tabular}{llccccc}
\toprule
Grid & Method & $K$ & Valid $\uparrow$ & Dist.\ $\uparrow$ & TV $\downarrow$ & Corr.(c) $\uparrow$ \\
\midrule
\multirow{5}{*}{$4{\times}4$} & SymmDiff (sampled) & 20 & 0.1 & 0.2 & 0.505 & 43.4 \\
 & PermFlow & 100 & 92.5{\scriptsize$\pm$0.5} & 90.7{\scriptsize$\pm$0.3} & 0.106{\scriptsize$\pm$0.000} & 98.8{\scriptsize$\pm$0.1} \\
 & \quad state-independent & 100 & 94.9 & 92.5 & 0.099 & 99.2 \\
 & \quad single-label coupling & 100 & 93.0 & 91.2 & 0.109 & 98.9 \\
 & \quad no projection & 100 & 91.6 & 90.3 & 0.109 & 98.6 \\
\midrule
$5{\times}5$ & PermFlow & 100 & 72.9{\scriptsize$\pm$0.3} & 73.4{\scriptsize$\pm$0.1} & 0.163{\scriptsize$\pm$0.001} & 96.1{\scriptsize$\pm$0.1} \\
$6{\times}6$ & PermFlow & 100 & 39.2{\scriptsize$\pm$1.8} & 39.8{\scriptsize$\pm$1.8} & 0.223{\scriptsize$\pm$0.003} & 89.2{\scriptsize$\pm$0.4} \\
\bottomrule
\end{tabular}
\vspace{-3pt}
\end{table}

Table~\ref{tab:jigsaw} reports the published setting, in which we compare with the published results of Gumbel-Sinkhorn~\citep{mena2018learning}, DiffSort~\citep{petersen2022monotonic}, EF-DiffSort~\citep{kim2024generalized} and SymmetricDiffusers (SymmDiff)~\citep{zhang2025symmetricdiffusers}. PermFlow matches or improves on every published result while decoding one unassisted sample
against beam-searched baselines, raising patch-level Correct from
$58.5\%$, $37.9\%$ and $18.5\%$ to $63.5\%$, $53.2\%$ and $41.1\%$. The
margin widens as the permutation space grows, so the advantage of PermFlow
is largest on the hardest puzzles.

Table~\ref{tab:orbit} reports the orbit arm. Here the benchmark's Accuracy
is no longer informative, since it counts every reassembly other than the
labelled one as wrong, even when it is pixel-identical to the original
image. We therefore measure whether a single sample is a perfect
reassembly (Valid), how many distinct perfect reassemblies $K$ samples
contain (Dist.), and how evenly each patch is spread over the slots it may
occupy (TV to uniform). PermFlow reassembles $92.5\%$ of $4\times4$ images
exactly from a single sample, $90.7\%$ of 100 samples are distinct perfect
reassemblies, and the per-patch distribution over admissible slots is close
to uniform, with a TV of $0.106$ against $0.098$ for an exactly uniform
sampler with the same number of samples. PermFlow thus does not merely find
one correct reassembly, but spreads its samples evenly over the many
equivalent ones, as expected from the nearest-member coupling.
SymmetricDiffusers rarely produces valid reassemblies. Across the ablations, validity changes by at most 2.4 percentage points, suggesting that the main challenge is placing the unique patches. Symmetric noise appears sufficient to recover the multiple solutions related by symmetry.  We refer to Appendix~\ref{app:structure} and ~\ref{app:jigsaw} for more details.


\subsection{Correlated Ranking Distributions: Mallows Mixtures and Sushi}
\label{sec:mallows}

\begin{figure}[htbp]
\centering
\providecommand{\rkA}{red!65}
\providecommand{\rkB}{orange!85}
\providecommand{\rkC}{yellow!85!orange}
\providecommand{\rkD}{teal!60}
\providecommand{\rkE}{blue!55}
\providecommand{\rankrow}[4]{
  \foreach \c [count=\i] in {#3}{%
    \fill[\c, draw=black!60, thin, opacity=#4]
      ({#1+(\i-1)*0.30},{#2}) rectangle ++(0.26,0.26);
  }%
}
\resizebox{\linewidth}{!}{%
\begin{tikzpicture}[font=\normalsize]
\node[font=\normalsize\bfseries, anchor=west] at (0,3.25) {(a) Target: conditional Mallows mixture};
\node[draw, rounded corners=2pt, fill=gray!10, align=center, inner sep=3pt, font=\small]
  (ctx) at (0.75,1.45) {context $\mathbf{x}$\\[-1pt]{\footnotesize synthetic: $\mathbf{x}\in\mathbb{R}^8$}\\[-1pt]{\footnotesize Sushi: respondent}};
\rankrow{2.05}{2.25}{\rkA,\rkB,\rkC,\rkD,\rkE}{1}
\rankrow{2.05}{1.90}{\rkA,\rkC,\rkB,\rkD,\rkE}{0.45}
\rankrow{2.05}{0.75}{\rkE,\rkD,\rkC,\rkB,\rkA}{1}
\rankrow{2.05}{0.40}{\rkE,\rkD,\rkB,\rkC,\rkA}{0.45}
\draw[-{Stealth[length=4pt]}, thick] (ctx.east) -- (1.98,2.30);
\draw[-{Stealth[length=4pt]}, thick] (ctx.east) -- (1.98,0.85);
\node[anchor=west, font=\small] at (3.6,2.30) {$\sigma_1(\mathbf{x})$, $w_1{=}0.7$};
\node[anchor=west, font=\small] at (3.6,0.85) {$\sigma_2(\mathbf{x})$, $w_2{=}0.3$};
\node[anchor=west, font=\footnotesize, text=black!60] at (2.0,1.55) {+ Mallows noise (few swaps)};
\node[anchor=west, font=\footnotesize, text=black!60] at (2.0,0.08) {$\sigma_2$ = reversal of $\sigma_1$};
\begin{scope}[xshift=6.3cm]
\node[font=\normalsize\bfseries, anchor=west] at (0,3.25) {(b) One perturbed score (PL, GS)};
\draw[black!50] (0,1.0) -- (1.55,1.0);
\foreach \c/\h [count=\i] in {\rkA/0.62,\rkB/0.55,\rkC/0.50,\rkD/0.55,\rkE/0.60}{
  \fill[\c, draw=black!60, thin] ({(\i-1)*0.30},1.0) rectangle ++(0.26,\h);}
\node[font=\footnotesize, align=center] at (0.75,0.62) {scores $s(\mathbf{x})$\\[-2pt] nearly flat};
\draw[-{Stealth[length=4pt]}, thick] (1.7,1.35) -- node[above, font=\footnotesize]{$+$\,Gumbel} node[below, font=\footnotesize]{sort} (2.75,1.35);
\rankrow{2.9}{2.25}{\rkC,\rkA,\rkE,\rkB,\rkD}{1}
\rankrow{2.9}{1.75}{\rkB,\rkD,\rkA,\rkE,\rkC}{1}
\rankrow{2.9}{1.25}{\rkD,\rkC,\rkB,\rkA,\rkE}{1}
\rankrow{2.9}{0.75}{\rkA,\rkE,\rkC,\rkD,\rkB}{1}
\node[font=\footnotesize, align=center, text=red!70!black] at (3.6,0.25) {samples fall between\\[-2pt]the two modes};
\end{scope}
\begin{scope}[xshift=11.55cm]
\node[font=\normalsize\bfseries, anchor=west] at (0,3.25) {(c) PermFlow: transport to both modes};
\foreach \px/\py in {0.35/1.55,0.5/1.30,0.28/1.20,0.62/1.62,0.45/1.45,0.70/1.38,0.38/1.70,0.58/1.18,0.25/1.42,0.66/1.50}{
  \fill[gray!70] (\px,\py) circle (1.3pt);}
\node[font=\footnotesize, align=center] at (0.48,2.05) {noise $\tilde X_0\in\mathcal{B}_a$};
\foreach \py/\ty in {1.70/2.38,1.62/2.36,1.55/2.34,1.50/2.32,1.45/2.30,1.42/2.40,1.38/2.28}{
  \draw[-{Stealth[length=3pt]}, black!55, thin] (0.78,\py) .. controls (1.4,\py) and (1.6,\ty) .. (2.15,\ty);}
\foreach \py/\ty in {1.30/0.88,1.20/0.84,1.18/0.80}{
  \draw[-{Stealth[length=3pt]}, black!55, thin] (0.78,\py) .. controls (1.4,\py) and (1.6,\ty) .. (2.15,\ty);}
\rankrow{2.25}{2.20}{\rkA,\rkB,\rkC,\rkD,\rkE}{1}
\rankrow{2.25}{0.70}{\rkE,\rkD,\rkC,\rkB,\rkA}{1}
\node[anchor=west, font=\small] at (3.80,2.33) {$\approx70\%$};
\node[anchor=west, font=\small] at (3.80,0.83) {$\approx30\%$};
\node[font=\footnotesize, align=center, text=black!60] at (1.75,0.25) {$S$ Euler steps,\\[-2pt]then Hungarian rounding};
\end{scope}
\end{tikzpicture}}
\caption{The sample-valued regime. (a)~A context (synthetic features, or a
Sushi respondent's attributes) induces a mixture of rankings concentrated
around two reversed centres with context-dependent weights; each training
example is one ranking drawn from it. (b)~Plackett--Luce and
Gumbel-Sinkhorn perturb a single score object with i.i.d.\ noise. For two
reversed modes the best single score is nearly flat, and the samples fall
between the modes rather than on them. (c)~PermFlow, trained with the independent coupling, learns to route
different noise draws to different modes, and its samples split between
the modes in close to the right proportions, as shown in Table~\ref{tab:mallows}.}
\label{fig:mallows}
\vspace{-3mm}
\end{figure}

In the sample-valued regime each training input comes with one ranking
drawn from an unknown, heterogeneous distribution, as illustrated in Figure~\ref{fig:mallows}.
In the synthetic task a context $\mathbf{x}\in\R^8$ determines centre
rankings $\sigma_m(\mathbf{x})$ and weights $w(\mathbf{x})$, and a ranking
is drawn from $\sum_m w_m(\mathbf{x})\,\mathrm{Mallows}(\sigma_m(\mathbf{x}),1)$,
either with two reversed centres, which no single Plackett--Luce model can
represent because its distribution has a single local mode, or with three
random ones.
For $n\le8$ the exact distribution is enumerated, and we report the total
variation (TV) of $K=1000$ samples to it, the \emph{mode} TV (error in the
proportions of samples nearest to each centre) and the pairwise preference
error. On Sushi~\citep{kamishima2003nantonac}, 5{,}000 respondents ranked
10 sushi; we generate rankings for 1{,}000 held-out respondents from their
attributes and compare population statistics with the observed rankings.
PermFlow uses the independent coupling, with the state token of Section~\ref{sec:experiments} or, in the ``rows'' variant, additionally each row of $X_t$ added to its item token; the baselines GS,
Plackett--Luce (PL) and a mixture of four PL models (a flexible parametric
baseline) share the backbone. In Table~\ref{tab:mallows} PermFlow
integrates with $S=50$ Euler steps, a choice analysed in
Section~\ref{sec:steps}.

\begin{table}[htbp]
\centering
\small
\caption{Conditional Mallows mixtures (lower is better). On the reversed
mixtures PermFlow reduces the mode error of GS and PL several-fold; the PL
mixture is the strongest method. rev./rand.: two reversed or three random
centres.}
\label{tab:mallows}
\setlength{\tabcolsep}{4pt}
\begin{tabular}{l|cccc|cccc}
\toprule
& \multicolumn{4}{c|}{Mode TV $\downarrow$} & \multicolumn{4}{c}{Pairwise error $\downarrow$} \\
Method & $6$, rev. & $8$, rev. & $8$, rand. & $10$, rev. & $6$, rev. & $8$, rev. & $8$, rand. & $10$, rev. \\
\midrule
\textit{Exact samples (floor)} & 0.008 & 0.008 & 0.012 & 0.009 & 0.011 & 0.010 & 0.010 & 0.011 \\
\midrule
Plackett--Luce & 0.066 & 0.094 & 0.093 & 0.132 & 0.039 & 0.047 & 0.056 & 0.060 \\
Gumbel-Sinkhorn & 0.096 & 0.103 & 0.136 & 0.135 & 0.133 & 0.118 & 0.106 & 0.120 \\
\midrule
PermFlow & 0.025 & 0.041 & 0.106 & 0.038 & 0.043 & 0.054 & 0.083 & 0.048 \\
PermFlow (rows) & 0.017 & 0.029 & 0.051 & 0.031 & 0.033 & 0.042 & 0.056 & 0.040 \\
PL mixture ($M'{=}4$) & \textbf{0.011} & \textbf{0.010} & \textbf{0.017} & \textbf{0.013} & \textbf{0.016} & \textbf{0.016} & \textbf{0.027} & \textbf{0.017} \\
\bottomrule
\end{tabular}
\vspace{-3mm}
\end{table}

\begin{table}[t]
\centering
\small
\caption{Sushi rankings of held-out respondents. PermFlow and the PL mixture
match the population statistics far better than GS
and PL. Right: Kendall-$\tau$ to the respondent's own ranking.}
\label{tab:sushi}
\setlength{\tabcolsep}{4pt}
\begin{tabular}{l|cccc|cc}
\toprule
Method & Top-1 TV & Position & Pairwise & Top-3 TV & $\tau$ mean & $\tau$ best $\uparrow$ \\
\midrule
\textit{Training rankings (floor)} & 0.035 & 0.0091 & 0.012 & 0.295 & -- & -- \\
\midrule
Plackett--Luce & 0.169 & 0.0216 & 0.038 & 0.479 & 0.149 & 0.648 \\
Gumbel-Sinkhorn & 0.182 & 0.0238 & 0.058 & 0.417 & 0.215 & 0.650 \\
PermFlow & 0.060 & 0.0124 & \textbf{0.019} & 0.390 & 0.148 & \textbf{0.688} \\
PermFlow (rows) & \textbf{0.051} & \textbf{0.0115} & 0.022 & 0.434 & 0.144 & 0.676 \\
PL mixture ($M'{=}4$) & 0.070 & 0.0131 & 0.021 & \textbf{0.382} & 0.155 & 0.667 \\
\bottomrule
\end{tabular}
\end{table}

As shown in Table~\ref{tab:mallows}, on the reversed mixtures GS and PL cannot hold
mass on two opposite rankings, with a mode error of $0.07$--$0.14$, and PL
places a large share of its samples between the modes. PermFlow reduces the
mode error to $0.017$--$0.031$, three to six times lower than GS and PL. The
PL mixture, whose components can sit on the two centres, is stronger still
on these synthetic data (mode error $0.010$--$0.013$), and PermFlow ranks
second to it on every reversed mixture, also in pairwise error and in full TV. With three
random centres the state-token PermFlow is no better than PL in mode error,
whereas the rows variant halves it ($0.051$ against $0.093$).

Table~\ref{tab:sushi} reports the Sushi results. On this dataset, where no family is known to be correct,
both single-score baselines misrepresent the population (top-item TV
$0.17$--$0.18$, against $0.051$--$0.060$ for PermFlow and a data floor of
$0.035$), while PermFlow and the PL mixture are on par: PermFlow is better
on the top item, position marginals and pairwise preferences and has the
highest best-of-100 Kendall-$\tau$, the PL mixture slightly better on the
top three. The mean Kendall-$\tau$ rewards concentrating on one ranking and
is not used for comparison. We refer to Appendix~\ref{app:mallows} for the
data, models and evaluation protocol, full TV and standard deviations, and
to Appendix~\ref{app:structure} for the single-mode property of
Plackett--Luce.

\paragraph{How Many Integration Steps? A Cross-Task Analysis}
\label{sec:steps}

The two regimes suggest opposite behaviour of the integrator. With
set-valued supervision the optimal field, in its representative that
extends continuously to $t=0$, is straight within each cell: every noise
sample travels on a straight line to the target of its cell, Euler
integration is exact for every $S$, and a single step lands on the
target. With the independent coupling the start carries no information
about the target, so a field fitted near $t=0$ points from the start toward
the marginal $\mathbb{E}[P\mid\mathbf{c}]$, which rounds to essentially one
permutation per input; any diversity has to be produced later along the
trajectory. The latter is a heuristic rather than a theorem for our
spherical source, and we test both
expectations empirically.
Table~\ref{tab:steps} evaluates the same trained models with
$S\in\{1,\dots,50\}$.

\begin{table}[htbp]
\centering
\scriptsize
\caption{Number of Euler steps $S$ at inference. One step suffices with
set-valued supervision, whereas sample-valued supervision needs several.}
\label{tab:steps}
\setlength{\tabcolsep}{3.5pt}
\resizebox{\linewidth}{!}{%
\begin{tabular}{c|cc|ccc|ccc}
\toprule
& \multicolumn{5}{c|}{Set-valued (coupling cells)} & \multicolumn{3}{c}{Sample-valued (independent coupling)} \\
& \multicolumn{2}{c|}{Sorting: cal.\ err.\ $\downarrow$} & \multicolumn{2}{c}{Jigsaw $\sigma{=}0$: valid $\uparrow$} & Jigsaw $\sigma{=}0.01$ & \multicolumn{2}{c}{Mallows $n{=}8$: mode TV $\downarrow$} & Sushi \\
$S$ & nearest & calibrated & $4{\times}4$ & $6{\times}6$ & $6{\times}6$ Correct $\uparrow$ & 2 reversed & 3 random & top-1 TV $\downarrow$ \\
\midrule
1 & 0.178 & 0.064 & 93.5 & 46.9 & 42.8 & 0.149 & 0.194 & 0.395 \\
2 & 0.174 & 0.065 & 93.2 & 42.5 & 42.0 & 0.122 & 0.124 & 0.132 \\
5 & 0.175 & 0.063 & 92.9 & 41.5 & 41.5 & 0.067 & 0.073 & 0.061 \\
10 & 0.176 & 0.063 & 92.7 & 40.8 & 41.2 & 0.050 & 0.059 & 0.060 \\
20 & 0.176 & 0.063 & 92.6 & 40.2 & 41.0 & 0.037 & 0.053 & 0.063 \\
50 & 0.176 & 0.063 & 92.6 & 39.8 & 40.8 & 0.029 & 0.051 & 0.060 \\
\bottomrule
\end{tabular}}
\end{table}

Both expectations are borne out. On sorting the calibration error varies by at most
$0.004$ across $S$, and on jigsaw puzzles one step is the best setting at
every grid, increasingly so with $n$ (at $6\times6$, $\sigma=0$, validity
is $46.9\%$ with one step and $40.8\%$ with ten), because the learned field
is least accurate at $t=0$, where the state is pure noise, and later
steps evaluate it on blurred states never seen in training. On the ranking tasks
one step collapses the samples onto the rounded marginal (mode error
$0.149$; Sushi top-1 TV $0.395$), five steps reduce these errors by
factors of $2.2$ and $6.5$, and the reversed mixture keeps improving up to
fifty steps. The step count should therefore follow the coupling: one step
with set-valued supervision, where a sample costs a single network
evaluation and the flow contributes a learnable partition of the noise,
and about ten with sample-valued supervision, where the learned field has
to create the multimodality along the trajectory. We refer to
Appendix~\ref{app:steps} for the metric and seeds of each column and the
per-task tables, to Appendix~\ref{app:jigsaw} for the analysis of the
jigsaw degradation, and to Appendix~\ref{app:unbiased} for the exactness of
Euler integration under set-valued supervision.

\section{Conclusion}
\label{sec:conclusion}
We presented PermFlow, a flow matching framework that learns
distributions over permutations by flowing directly on the affine
subspace of matrices with unit row and column sums. A closed-form tangent-space projector preserves unit row and column sums along every trajectory, while a nearest-target coupling assigns different regions of the initialization space to different valid permutations.
Empirically, PermFlow recovers multiple valid permutations under
ambiguity on visual sorting, jigsaw puzzles, and Mallows mixtures and
Sushi rankings, where Sinkhorn-based baselines fail structurally rather
than from poor tuning. The results suggest geometry-aware flow matching
as a viable alternative to entropic relaxation for combinatorial
prediction problems where multiple valid solutions coexist.

\section{Acknowledgments}
This project is supported by an AI2050 Senior Fellowship, a Schmidt Sciences program; the National Science Foundation(NSF); the National
Institute of Food and Agriculture (USDA/NIFA)(2023-67021-39829); the Air Force Office of Scientific Research (AFOSR)(FA9550-23-1-0322).
\bibliography{reference}
\newpage

\section*{NeurIPS Paper Checklist}

\begin{enumerate}

\item {\bf Claims}
    \item[] Question: Do the main claims made in the abstract and introduction accurately reflect the paper's contributions and scope?
    \item[] Answer: \answerYes{}
    \item[] Justification: The abstract and introduction state three contributions, a CFM framework on the affine subspace of doubly-stochastic matrices with exact constraint preservation, application to ambiguous permutation prediction, and an empirical study against Sinkhorn-based baselines.

\item {\bf Limitations}
    \item[] Question: Does the paper discuss the limitations of the work performed by the authors?
    \item[] Answer: \answerYes{}
    \item[] Justification: The conclusion explicitly discusses this.

\item {\bf Theory assumptions and proofs}
    \item[] Question: For each theoretical result, does the paper provide the full set of assumptions and a complete (and correct) proof?
    \item[] Answer: \answerYes{}
    \item[] Justification: Theorem 1 (geometric correctness of PermFlow) is fully proved in Appendix~\ref{app:proof} via the sandwich identity $\mathcal{C}(U) = (I-J)U(I-J)$, with three explicit steps establishing idempotency, image characterization, and self-adjointness, followed by the flow-invariance argument. 

    \item {\bf Experimental result reproducibility}
    \item[] Question: Does the paper fully disclose all the information needed to reproduce the main experimental results of the paper to the extent that it affects the main claims and/or conclusions of the paper (regardless of whether the code and data are provided or not)?
    \item[] Answer: \answerYes{}
    \item[] Justification: Algorithm~\ref{alg:general-perm-cfm} specifies the full training and inference procedure, including the projector, coupling, interpolation, and Euler integration steps. 

\item {\bf Open access to data and code}
    \item[] Question: Does the paper provide open access to the data and code, with sufficient instructions to faithfully reproduce the main experimental results, as described in supplemental material?
    \item[] Answer: \answerNo{}
    \item[] Justification: Code is now released at: https://github.com/yimengmin/PermFlow/.

\item {\bf Experimental setting/details}
    \item[] Question: Does the paper specify all the training and test details (e.g., data splits, hyperparameters, how they were chosen, type of optimizer) necessary to understand the results?
    \item[] Answer: \answerYes{}
    \item[] Justification: Appendix specify architectures.

\item {\bf Experiment statistical significance}
    \item[] Question: Does the paper report error bars suitably and correctly defined or other appropriate information about the statistical significance of the experiments?
    \item[] Answer: \answerNo{}
    \item[] Justification: Results are reported with std for the noisy mnist but not the puzzle and Mallows Mixtures and Sushi.

\item {\bf Experiments compute resources}
    \item[] Question: For each experiment, does the paper provide sufficient information on the computer resources (type of compute workers, memory, time of execution) needed to reproduce the experiments?
    \item[] Answer: \answerYes{}
    \item[] Justification: Appendix reports that the models' parameters and the setup.
    
\item {\bf Code of ethics}
    \item[] Question: Does the research conducted in the paper conform, in every respect, with the NeurIPS Code of Ethics \url{https://neurips.cc/public/EthicsGuidelines}?
    \item[] Answer: \answerYes{}
    \item[] Justification: The work uses only the publicly available  dataset and synthetically generated cost matrices, involves no human subjects or sensitive data, and conforms with the NeurIPS Code of Ethics.

\item {\bf Broader impacts}
    \item[] Question: Does the paper discuss both potential positive societal impacts and negative societal impacts of the work performed?
    \item[] Answer: \answerNA{}
    \item[] Justification: This work is foundational research on flow-based generative modeling for permutation distributions, evaluated on synthetic and standard benchmark tasks. We do not see a direct path to specific positive or negative societal impacts beyond those that apply to combinatorial optimization and generative modeling broadly.
    
\item {\bf Safeguards}
    \item[] Question: Does the paper describe safeguards that have been put in place for responsible release of data or models that have a high risk for misuse (e.g., pre-trained language models, image generators, or scraped datasets)?
    \item[] Answer: \answerNA{}
    \item[] Justification: The paper does not release pretrained models, scraped datasets, or other assets with foreseeable misuse risk.

\item {\bf Licenses for existing assets}
    \item[] Question: Are the creators or original owners of assets (e.g., code, data, models), used in the paper, properly credited and are the license and terms of use explicitly mentioned and properly respected?
    \item[] Answer: \answerYes{}
    \item[] Justification: The datasets are cited and used under its standard public license. Methods compared against (Sinkhorn, Gumbel-Sinkhorn) and the flow matching framework on which we build are credited via citation.

\item {\bf New assets}
    \item[] Question: Are new assets introduced in the paper well documented and is the documentation provided alongside the assets?
    \item[] Answer: \answerNA{}
    \item[] Justification: No new assets are released with this submission. 

\item {\bf Crowdsourcing and research with human subjects}
    \item[] Question: For crowdsourcing experiments and research with human subjects, does the paper include the full text of instructions given to participants and screenshots, if applicable, as well as details about compensation (if any)? 
    \item[] Answer: \answerNA{}
    \item[] Justification: The paper does not involve crowdsourcing or human subjects.

\item {\bf Institutional review board (IRB) approvals or equivalent for research with human subjects}
    \item[] Question: Does the paper describe potential risks incurred by study participants, whether such risks were disclosed to the subjects, and whether Institutional Review Board (IRB) approvals (or an equivalent approval/review based on the requirements of your country or institution) were obtained?
    \item[] Answer: \answerNA{}
    \item[] Justification: The paper does not involve human subjects research, so IRB approval is not applicable.

\item {\bf Declaration of LLM usage}
    \item[] Question: Does the paper describe the usage of LLMs if it is an important, original, or non-standard component of the core methods in this research? Note that if the LLM is used only for writing, editing, or formatting purposes and does \emph{not} impact the core methodology, scientific rigor, or originality of the research, declaration is not required.
    \item[] Answer: \answerNA{}
    \item[] Justification: LLMs were not used as a component of the core methodology. Any use was limited to writing and editing, which does not require declaration per the NeurIPS LLM policy.

\end{enumerate}
\bibliographystyle{plainnat}

\newpage
\appendix
\section{Extended Background}
\label{app:background}

\subsection*{Sinkhorn Iteration for Soft Permutations}

Given a score matrix $S \in \mathbb{R}^{n \times n}$, the goal of permutation prediction is often formulated as a linear assignment problem
\begin{equation}
\max_{P \in \mathcal{P}_n} \langle S, P \rangle ,
\end{equation}
where $\mathcal{P}_n$ is the set of $n\times n$ permutation matrices. Because $\mathcal{P}_n$ is discrete, a common relaxation replaces it with the Birkhoff polytope
\begin{equation}
\mathcal{B}_+ =
\{X \in \mathbb{R}_+^{n\times n} \mid
X\mathbf{1}=\mathbf{1},\;
X^\top\mathbf{1}=\mathbf{1}\},
\end{equation}
the set of doubly stochastic matrices. To obtain a differentiable solution, entropy regularization is introduced:
\begin{equation}
\max_{X\in\mathcal{B}_+}
\langle S,X\rangle + \tau H(X),
\end{equation}
where $H(X)$ is the Shannon entropy and $\tau>0$ is a temperature parameter. The solution takes the form
\begin{equation}
X^\star = \mathrm{Diag}(u)\,K\,\mathrm{Diag}(v),
\qquad
K=\exp(S/\tau),
\end{equation}
where scaling vectors $u$ and $v$ enforce row and column normalization.  
These vectors can be computed efficiently by alternating row and column normalization~\citep{sinkhorn1964relationship}. In practice, the resulting matrix $X$ is a soft permutation (doubly stochastic matrix).  
As $\tau \to 0$, the solution becomes closer to a permutation matrix but gradients become unstable. Larger values of $\tau$ produce smoother but less discrete solutions.

For a fixed score matrix, deterministic Sinkhorn returns a single doubly
stochastic matrix, so after rounding it outputs one permutation per input.
Gumbel-Sinkhorn~\citep{mena2018learning} instead perturbs the scores
with i.i.d.\ $G_{ij}\sim\mathrm{Gumbel}(0,1)$ and computes
$\mathrm{Sinkhorn}((S+G)/\tau)$. Two decoders are in use. The \emph{soft}
decoder rounds the Sinkhorn output, $\mathrm{Round}(\mathrm{Sinkhorn}((S+G)/\tau))$;
the \emph{Gumbel-Matching} decoder rounds the perturbed scores,
$\mathrm{Round}(S+G)$. They are not equivalent in general: Sinkhorn
scaling adds row and column constants to the \emph{logarithm} of the
matrix, so $\mathrm{Round}(\log\mathrm{Sinkhorn}((S+G)/\tau))=\mathrm{Round}(S+G)$,
but Hungarian rounding maximises the sum of the selected entries of the
matrix itself, and taking logarithms changes that objective. For example,
the doubly stochastic matrix with rows $(0.2226,0.3396,0.4379)$,
$(0.0171,0.4385,0.5444)$, $(0.7603,0.2220,0.0178)$ is rounded to the
column assignment $(2,3,1)$, whereas its entrywise logarithm is rounded to
$(3,2,1)$. We state which decoder is used wherever GS results are
reported.

\subsection*{Plackett--Luce and Mallows Models}
The Plackett--Luce (PL) model~\citep{plackett1975analysis,luce1959individual}
with item scores $s\in\R^n$ assigns to an order
$(\pi_1,\ldots,\pi_n)$, best first, the probability
\begin{equation}
p_{\mathrm{PL}}(\pi\mid s)=\prod_{k=1}^{n}
\frac{\exp(s_{\pi_k})}{\sum_{l\ge k}\exp(s_{\pi_l})},
\end{equation}
and is sampled by sorting $s+G$ with i.i.d.\ Gumbel $G$. The Mallows
model~\citep{mallows1957non} with centre $\sigma$ and dispersion
$\theta>0$ assigns $p(\pi)\propto\exp(-\theta\,d_K(\pi,\sigma))$, where
$d_K$ is the Kendall distance (the number of discordant item pairs); it is
sampled exactly by the repeated insertion model. Both are unimodal
families, and mixtures of them are used for heterogeneous populations.

\subsection*{Flow Matching and Conditional Flow Matching}

Flow Matching is a generative framework that learns a continuous transport process transforming samples from a source distribution $p_0$ into samples from a target distribution $p_1$.
Given noisy samples $\mathbf{x}_0 \sim p_0$ and clean samples $\mathbf{x}_1 \sim p_1$, the goal is to learn a time-dependent velocity field
$
u_t(\mathbf{x}_t)
$ that transports probability mass from $p_0$ to $p_1$. Let $\psi_t(\mathbf{x}_0 \mid \mathbf{x}_1)$ denote an interpolant that defines the trajectory between $\mathbf{x}_0$ and $\mathbf{x}_1$. The interpolant satisfies the boundary conditions
\begin{equation}
\psi_0(\mathbf{x}_0 \mid \mathbf{x}_1) = \mathbf{x}_0,
\qquad
\psi_1(\mathbf{x}_0 \mid \mathbf{x}_1) = \mathbf{x}_1 .
\end{equation}

At time $t \in [0,1]$, the intermediate point is $
\mathbf{x}_t = \psi_t(\mathbf{x}_0 \mid \mathbf{x}_1).$ The conditional velocity field is defined as the time derivative of the interpolant:
\begin{equation}
u_t(\mathbf{x}_t \mid \mathbf{x}_1)
=
\frac{d}{dt}\psi_t(\mathbf{x}_0 \mid \mathbf{x}_1).
\end{equation}

For a velocity field $u_t$ to generate the evolving distribution $p_t$, it must satisfy the continuity equation
\begin{equation}
\frac{\partial}{\partial t} p_t(\mathbf{x})
=
-\nabla \cdot \big(p_t(\mathbf{x}) u_t(\mathbf{x})\big),
\end{equation}
where $\nabla \cdot$ denotes the divergence operator. Flow matching objective trains a parameterized velocity model $u_\theta(\mathbf{x}_t,t)$ to approximate the true marginal velocity field:
\begin{equation}
\mathcal{L}_\text{FM}
=
\mathbb{E}_{t,\mathbf{x}_t}
\left[
\left\|
u_\theta(\mathbf{x}_t,t) - u_t(\mathbf{x}_t)
\right\|^2
\right].
\end{equation}

However, computing the marginal velocity $u_t(\mathbf{x}_t)$ requires integrating over all trajectories connecting $p_0$ and $p_1$, which is generally intractable. To address this issue, CFM conditions the velocity field on individual data samples $\mathbf{x}_1$. The resulting CFM objective becomes
\begin{equation}
\mathcal{L}_\text{CFM}
=
\mathbb{E}_{t,\mathbf{x}_0,\mathbf{x}_1}
\left[
\left\|
u_\theta(\mathbf{x}_t,t) -
u_t(\mathbf{x}_t \mid \mathbf{x}_1)
\right\|^2
\right].
\end{equation}

Importantly, the conditional and unconditional objectives share the same gradient:
\begin{equation}
\nabla_\theta \mathcal{L}_\text{FM}
=
\nabla_\theta \mathcal{L}_\text{CFM},
\end{equation}

making the conditional formulation both tractable and theoretically
consistent~\citep{lipmanflow}. The pairing of $\mathbf{x}_0$ with
$\mathbf{x}_1$ is a free design choice: with the \emph{independent
coupling} $\mathbf{x}_0\perp\mathbf{x}_1$ the optimal field transports
$p_0$ to $p_1$, but conditional paths cross and the field is curved; with a
\emph{deterministic coupling} in which $\mathbf{x}_1$ is a function of
$\mathbf{x}_0$, paths can be made non-crossing and the field
straight~\citep{tong2023improving}.

\section{Proof of Theorem~\ref{thm:main}}
\label{app:proof}
 
The proof rests on a single algebraic identity. Let
\[
    J \;=\; \tfrac{1}{n} \1\1^\top
\]
be the rank-one averaging matrix. Note that $J$ is symmetric, $J^2 = J$
(so $J$ is itself an orthogonal projector on $\R^n$), and $I - J$ is its
orthogonal complement projector.
 
\begin{lemma}[Sandwich identity]
\label{lem:sandwich}
For all $U \in \R^{n\times n}$,
\[
    \C(U) \;=\; (I - J)\, U \,(I - J).
\]
\end{lemma}
 
\begin{proof}
Direct expansion:
\[
    (I - J)\,U\,(I - J) \;=\; U - JU - UJ + JUJ.
\]
Substituting $J = \tfrac{1}{n}\1\1^\top$ gives
\[
    JU = \tfrac{1}{n}\1\1^\top U, \qquad
    UJ = \tfrac{1}{n} U \1\1^\top, \qquad
    JUJ = \tfrac{1}{n^2}\1\1^\top U \1\1^\top,
\]
which recovers $\C(U)$ exactly.
\end{proof}
 
With Lemma~\ref{lem:sandwich} in hand, the proof of Theorem~\ref{thm:main}
proceeds in three short steps.
 
First, $\C$ is idempotent: since $(I - J)^2 = I - 2J + J^2 = I - J$,
\[
    \C^2(U) \;=\; (I-J)\bigl[(I-J)\, U\, (I-J)\bigr](I-J)
            \;=\; (I-J)^2 \, U \, (I-J)^2
            \;=\; (I-J)\, U\, (I-J) \;=\; \C(U).
\]
 
Second, $\im(\C) = T$ and $\C$ is the identity on $T$. For any $U \in \R^{n\times n}$, using $\1^\top(I - J) = 0$ and $(I-J)\1 = 0$:
\[
    \C(U)\,\1 \;=\; (I-J)\, U\,(I-J)\1 \;=\; 0, \qquad
    \1^\top\, \C(U) \;=\; \1^\top(I-J)\, U\,(I-J) \;=\; 0.
\]
Hence $\im(\C) \subseteq T$. Conversely, if $\Delta \in T$, then $\Delta\1 = 0$
implies $\Delta J = 0$, and $\1^\top \Delta = 0$ implies $J\Delta = 0$. So
\[
    \C(\Delta) \;=\; (I-J)\,\Delta\,(I-J) \;=\; \Delta - J\Delta - \Delta J + J\Delta J \;=\; \Delta.
\]
Therefore $T \subseteq \im(\C)$ and $\C$ acts as the identity on $T$.
 
Third, $\C$ is self-adjoint: for any $A, B \in \R^{n\times n}$, using the cyclic property of the trace
together with $(I - J)^\top = I - J$:
\[
    \inner{\C(A)}{B}
    = \tr\bigl((I-J)\,A\,(I-J)\, B^\top\bigr)
    = \tr\bigl(A\,(I-J)\, B^\top\,(I-J)\bigr)
    = \inner{A}{\C(B)}.
\]
 
A linear map that is idempotent and self-adjoint is, by
definition, an orthogonal projector. Combined with the image
characterization, this proves that $\C$ is the orthogonal
projector onto $T$.
 
Finally, for flow invariance let $X_t$ be a solution on its interval of
existence $[0,T^\ast)$, with $T^\ast\le\infty$; the projection does not by
itself guarantee that solutions exist for all times when $f_\theta$ is
arbitrary. For $t\in[0,T^\ast)$ define $\phi(t) := X_t\1 - \1$. Then
\[
    \frac{d}{dt}\phi(t) \;=\; \dot{X}_t\, \1 \;=\; \C\bigl(f_\theta(X_t, t)\bigr)\, \1 \;=\; 0
\]
by the second step. Since $\phi(0) = X_0\1 - \1 = 0$, we conclude $X_t \1 = \1$
for all $t\in[0,T^\ast)$. An identical argument applied to
$\psi(t) := X_t^\top \1 - \1$ gives $X_t^\top \1 = \1$. Hence $X_t \in \B_a$
for every $t\in[0,T^\ast)$. \hfill$\blacksquare$

\section{Affine Constraints, Rounding, and a Per-Instance Certificate}
\label{app:rounding}

Our guarantee is exact preservation of the affine equality constraints
\[
    \B_a \;=\; \{X \in \R^{n\times n} : X\1 = \1,\ X^\top\1 = \1\},
\]
throughout the trajectory. The nonnegativity inequalities are relaxed, so
iterates need not be doubly stochastic and are not permutation matrices;
$\B_a$ is the affine hull of the Birkhoff polytope $\B_+$, not $\B_+$
itself. The concern this raises is that relaxing nonnegativity might let
the terminal rounding silently change the decoded permutation. It does
not, and this appendix says exactly why.

\subsection{Rounding does not depend on nonnegativity}

The statements below hold for arbitrary real $X$. They use neither
$X \geq 0$, nor $X \in \B_+$, nor even $X \in \B_a$.

Since $\|Q\|_F^2 = n$ is constant over permutation matrices, expanding the
objective gives, for all $Q \in \mathcal{P}_n$,
\begin{equation}
    \|X - Q\|_F^2 \;=\; \|X\|_F^2 \;-\; 2\inner{X}{Q} \;+\; n .
\end{equation}
Minimising $\|X - Q\|_F$ is therefore identical to maximising
$\inner{X}{Q}$, which is exactly the linear assignment problem that
Hungarian rounding solves. Hungarian rounding thus returns
\begin{equation}
    \hat{Q} \;=\; \argmin_{Q \in \mathcal{P}_n} \|X - Q\|_F ,
\end{equation}
the nearest permutation matrix in Frobenius norm, negative entries
notwithstanding. The same conclusion can be reached from the invariances
of the assignment problem: its optimum is unchanged by adding a constant
to any row or column and by a global positive rescaling, so no sign
condition can affect which permutation is returned. The invariance is
specific to these two operations and does not extend to arbitrary
normalisations. Dividing each row by its sum, for example, can change the
assignment: for
\[
X=\begin{pmatrix}10&8\\3&2\end{pmatrix}
\]
the identity scores $12$ and the swap $11$, whereas after row
normalisation they score approximately $0.956$ and $1.044$.

\subsection{The certificate}

Combined with the fact that any two distinct permutation matrices differ
in at least two rows, so that $\|Q_1 - Q_2\|_F \geq 2$, we obtain the
following.

\begin{proposition}[Rounding certificate]
\label{prop:rounding}
Let $X \in \R^{n\times n}$ be the continuous endpoint of the flow. If
there exists a permutation matrix $Q^\star$ with
$\|X - Q^\star\|_F < 1$, then $Q^\star$ is the unique nearest permutation
matrix, Hungarian rounding returns exactly $Q^\star$, and the rounding
displaces the endpoint by at most $\|X - Q^\star\|_F$.
\end{proposition}

\begin{proof}
For any permutation matrix $Q \neq Q^\star$, the reverse triangle
inequality gives
\[
    \|X - Q\|_F \;\geq\; \|Q - Q^\star\|_F - \|X - Q^\star\|_F
    \;>\; 2 - 1 \;=\; 1 \;>\; \|X - Q^\star\|_F .
\]
Hence $Q^\star$ is the unique minimiser of $\|X - Q\|_F$ over
$\mathcal{P}_n$.
\end{proof}

The threshold is exactly half the minimum distance between distinct
permutation matrices, so it is the largest radius at which uniqueness can
be asserted without further information.

\begin{corollary}[The certificate also bounds infeasibility]
\label{cor:infeasibility}
Under the hypothesis of Proposition~\ref{prop:rounding},
\[
    \|\min(X, 0)\|_F \;\leq\; \|X - Q^\star\|_F \;<\; 1 .
\]
\end{corollary}

\begin{proof}
$Q^\star$ is entrywise nonnegative, so at any entry with $X_{ij} < 0$,
\[
    |X_{ij} - Q^\star_{ij}| \;=\; Q^\star_{ij} - X_{ij} \;\geq\; |X_{ij}| .
\]
Entries with $X_{ij} \geq 0$ contribute nothing to $\|\min(X,0)\|_F$.
Summing squares over the negative entries gives the claim.
\end{proof}

Corollary~\ref{cor:infeasibility} is the direct answer to the underlying
concern: whenever the certificate passes, the violated nonnegativity
constraints are violated by a controlled amount, and the rounding cannot
have altered the decoded permutation.

\subsection{How to use it}

The condition is checkable per instance. One computes $\hat{Q}$ by
Hungarian rounding, evaluates $\|X - \hat{Q}\|_F$, and if the value is
below $1$ the decoded permutation is provably the unique nearest
permutation matrix, so rounding cannot have changed the decision. The
check costs one Frobenius norm on top of the rounding already performed
at inference.

\begin{remark}
This is a conditional certificate evaluated at test time, not an a priori
guarantee. It does not by itself bound $\|X - Q^\star\|_F$ for an
arbitrary endpoint of the learned flow, and instances that fail the check
are not certified either way.
\end{remark}

\section{The Optimal Field under Set-Valued Supervision}
\label{app:unbiased}
This appendix states and proves Theorem~\ref{thm:unbiased} (Appendix~\ref{app:unbiased}). The
result concerns the population-optimal velocity field: which targets it
reaches, by which paths, and with what probability. It says nothing about
how closely a trained network approximates that field;
Section~\ref{sec:jigsaw} shows where the approximation is worst.

\subsection{Setup}
Fix a context and its valid targets
$\mathcal{T} = \{P^1, \ldots, P^M\} \subset \mathcal{P}_n$. Let
$X_0 = \tfrac{1}{n}\1\1^\top$, and let $\mu_0$ be the law of
$\tilde X_0 = X_0 + \sigma_0\varepsilon$ with $\varepsilon$ uniform on the
unit sphere of $T$. For weights $w \in \R^M$ define the margins
$g_{mj}(X)=\inner{X}{P^m-P^j}+w_m-w_j$ and the cells
\begin{equation}
    \mathcal{V}_m \;=\; \{X \in \B_a : g_{mj}(X)>0 \ \text{for all } j \ne m\},
\end{equation}
which are relatively open, convex (intersections of open half-spaces),
pairwise disjoint, and cover $\B_a$ up to a $\mu_0$-null set. Write $m(X)$
for the index of the cell containing $X$. Training draws
$\tilde X_0\sim\mu_0$, sets $P^\star=P^{m(\tilde X_0)}$, draws
$t\sim\mathcal{U}(0,1)$, and regresses $v(X_t,t)$ onto $P^\star-\tilde X_0$
along $X_t=(1-t)\tilde X_0+tP^\star$. We say that \emph{every target lies
in its own cell} if $P^m\in\mathcal{V}_m$ for all $m$.

We use two elementary facts: every permutation matrix $P$ satisfies
$\|P\|_F^2=n$ and $\inner{X_0}{P}=1$, and distinct permutation matrices
satisfy $\inner{P}{Q}\le n-2$, i.e.\ $\|P-Q\|_F\ge 2$.

\begin{lemma}[Training paths stay in their cell]
\label{lem:invariance}
If every target lies in its own cell and $\tilde X_0\in\mathcal{V}_m$, then
$X_t=(1-t)\tilde X_0+tP^m\in\mathcal{V}_m$ for every $t\in[0,1]$.
\end{lemma}
\begin{proof}
Each $g_{mj}$ is affine, so
$g_{mj}(X_t)=(1-t)\,g_{mj}(\tilde X_0)+t\,g_{mj}(P^m)$, a convex
combination of two positive numbers.
\end{proof}

\begin{lemma}[When targets lie in their own cells]
\label{lem:owncell}
(a) For the nearest-target coupling ($w\equiv0$), every target lies in its
own cell. (b) For two targets and the calibrated coupling, which selects
$P^a$ iff $\inner{X-X_0}{d}>\sigma_0q$ with
$d=(P^a-P^b)/\|P^a-P^b\|_F$ and $q\in(-1,1)$, both targets lie in their
own cells whenever $\sigma_0\le1$.
\end{lemma}
\begin{proof}
(a) $g_{mj}(P^m)=n-\inner{P^m}{P^j}\ge2>0$.
(b) $\inner{P^a-X_0}{d}=\tfrac12\|P^a-P^b\|_F\ge1>\sigma_0q$ because
$\inner{X_0}{P^a-P^b}=0$ and $\|P^a\|_F=\|P^b\|_F$; symmetrically
$\inner{P^b-X_0}{d}=-\tfrac12\|P^a-P^b\|_F\le-1<\sigma_0q$.
\end{proof}

\subsection{Main result}
\begin{theorem}[Straight optimal paths and mode probabilities]
\label{thm:unbiased}
Suppose every target lies in its own cell (Lemma~\ref{lem:owncell}), and define the field
\begin{equation}
\label{eq:vbar}
\bar v(X,t)\;=\;\frac{P^{m(X)}-X}{1-t},\qquad
X\in\textstyle\bigcup_m\mathcal{V}_m,\ t\in[0,1).
\end{equation}
Then:
\begin{enumerate}
\item[(i)] every measurable minimiser $v^\star:\B_a\times[0,1)\to T$ of the
  population objective
  $\mathbb{E}_{t,\tilde X_0}\|v(X_t,t)-(P^\star-\tilde X_0)\|_F^2$ coincides
  with $\bar v$ almost everywhere with respect to the law of the training
  state $(X_t,t)$; in particular $\bar v$ is itself a minimiser;
\item[(ii)] the solution of $\dot X_t=\bar v(X_t,t)$ from
  $\tilde X_0\in\mathcal{V}_m$ is the training path
  $X_t=(1-t)\tilde X_0+tP^m$, it stays in $\mathcal{V}_m$, and explicit
  Euler integration of $\bar v$ with step $h=1/S$ reproduces it exactly at
  every grid point, for every $S\ge1$;
\item[(iii)] the path extends continuously to $X_1:=\lim_{t\to1}X_t=P^m$,
  so $\mathrm{Round}(X_1)=P^m$ and
  \begin{equation}
    \Pr\bigl[\hat P = P^m\bigr] \;=\; \mu_0(\mathcal{V}_m) \qquad\text{for every } m .
  \end{equation}
\end{enumerate}
\end{theorem}
Statements (ii) and (iii) concern the representative $\bar v$, not an
arbitrary minimiser. The objective constrains a minimiser only almost
everywhere with respect to the law of $(X_t,t)$; since $t\sim\mathcal{U}(0,1)$,
the slice $t=0$, at which every Euler scheme makes its first evaluation,
has measure zero, and a minimiser may be modified there without changing
the loss. Among the minimisers, $\bar v$ is the one obtained by extending
the formula continuously in $t$ along each training path, including to
$t=0$; a learned field that is continuous in $t$ and close to a minimiser
on the support is close to $\bar v$ there as well.
\begin{proof}
(i) By Lemma~\ref{lem:invariance}, a training state $X_t=X$ produced from
a start in $\mathcal{V}_m$ lies in $\mathcal{V}_m$, and the cells are
disjoint, so $X$ determines the cell index $m=m(X)$ of the start that
produced it. Given $(X,t)$ with $t<1$ and $m$, the start is determined as
well, $\tilde X_0=(X-tP^m)/(1-t)$. The regression target is therefore a
deterministic function of $(X_t,t)$,
\[
P^m-\tilde X_0=P^m-\frac{X-tP^m}{1-t}=\frac{P^m-X}{1-t}=\bar v(X,t),
\]
so the conditional mean of the target given $(X_t,t)$ equals
$\bar v(X_t,t)$ almost surely. The squared loss is minimised exactly by
functions that agree with this conditional mean almost everywhere with
respect to the law of $(X_t,t)$, which proves both claims.

(ii) Along $X_t=(1-t)\tilde X_0+tP^m$ we have
$\dot X_t=P^m-\tilde X_0=(P^m-X_t)/(1-t)=\bar v(X_t,t)$, and $X_t\in\mathcal{V}_m$
by Lemma~\ref{lem:invariance}, where $\bar v$ is given by the same formula.
For Euler, suppose $X_{sh}-P^m=(1-sh)(\tilde X_0-P^m)$ (true for $s=0$). Then
\[
X_{(s+1)h}-P^m=\bigl(X_{sh}-P^m\bigr)\Bigl(1-\frac{h}{1-sh}\Bigr)
=(1-sh-h)(\tilde X_0-P^m),
\]
so the induction hypothesis holds at $s+1$, and every Euler iterate lies on
the training path; the evaluations use $\bar v$ at $t=sh<1$ only.

(iii) As $t\to1$ (or at $s=S$) the factor $1-t$ vanishes and the path
reaches $P^m$, which rounds to itself. The decoded permutation is $P^m$
exactly when $\tilde X_0\in\mathcal{V}_m$, an event of probability
$\mu_0(\mathcal{V}_m)$.
\end{proof}

\begin{corollary}[What the coupling buys]
\label{cor:coupling}
(a) With nearest-target cells, $\mu_0(\mathcal{V}_m)>0$ for every $m$, so
every valid permutation is sampled with positive probability. For $M=2$,
$\mu_0(\mathcal{V}_a)=\mu_0(\mathcal{V}_b)=\tfrac12$ exactly, by the
symmetry $\varepsilon\mapsto-\varepsilon$ and $\inner{X_0}{P^a-P^b}=0$:
nearest-target coupling is calibrated to no weight other than
$\tfrac12$.
(b) With the calibrated two-target coupling,
$\mu_0(\mathcal{V}_a)=\Pr[\inner{\varepsilon}{d}>q]=\pi_a$ by the choice of
$q$, so the optimal sampler reproduces the prescribed weights exactly.
(c) Let $\mathcal{T}=\mathcal{O}$ be the orbit of a permutation $P^1$
under a group $H$ of permutations of the positions,
$\mathcal{O}=\{P^1Q:Q\in H\}$ (for jigsaw puzzles, $H$ permutes the slots
within each class of identical patches). Then the nearest-target cells
all have mass $1/|\mathcal{O}|$, and the optimal sampler is uniform on
$\mathcal{O}$.
\end{corollary}
\begin{proof}
(a) and (b) are immediate from Theorem~\ref{thm:unbiased} and the
definitions. For (c), right multiplication by a permutation matrix $Q$
preserves $T$, the Frobenius norm and $X_0$, so $\mu_0$ is invariant under
$X\mapsto XQ$. Since $\inner{XQ}{PQ}=\inner{X}{P}$, the nearest-target
cell of $P^1Q$ is the image under $X\mapsto XQ$ of the cell of $P^1$, and
therefore has the same mass.
\end{proof}

\begin{remark}[One step for the representative field]
Part (ii) holds for $S=1$: a single Euler step from $\tilde X_0$ with
$\bar v$ lands exactly on the target. For a learned field this is a
statement about the continuous extension to $t=0$ of what the network
learns on $t\in(0,1)$, which the loss itself does not pin down; whether one
step is accurate in practice is an empirical question, examined in
Section~\ref{sec:steps}.
\end{remark}

\begin{remark}[The independent coupling and the spherical source]
If the target is drawn independently of $\tilde X_0$, part (i) fails and the
analysis above does not apply. The standard flow matching result, that the
conditional-expectation field transports the source law to the target law
\citep{lipmanflow,tong2023improving}, is usually stated for a source with a
density, and our source is not of that kind: $\tilde X_0$ is uniform on a
sphere of dimension $(n-1)^2-1$ in $\B_a$. For $t\in(0,1)$ the training
states generated by a target $P$ lie on the sphere of radius $(1-t)\sigma_0$
centred at $(1-t)X_0+tP$; spheres belonging to different targets intersect
only in lower-dimensional sets, so almost every training state identifies
its target and the regression target is not an average, while at $t=0$ all
spheres coincide and the target is $\mathbb{E}[P\mid\mathbf{c}]-\tilde X_0$.
The population objective therefore constrains the field only on a
measure-zero union of spheres, and does not by itself determine a flow
between them. We accordingly make no theoretical claim for this regime: the
description of curved trajectories in Sections~\ref{sec:mallows}
and~\ref{sec:steps} is an interpretation of the trained network, whose
smooth interpolation off the training support is what the experiments
measure. A source with a density on $\B_a$, such as an unnormalised
Gaussian in $T$, would bring the standard result into play; for large
$(n-1)^2$ the normalised source is close to such a Gaussian at the
corresponding radius, but we do not rely on this approximation.
\end{remark}

\subsection{A state-independent special case}
\label{app:stateindep}
If the network ignores $(X_t,t)$, the velocity $v_\theta(\mathbf{c})$ is
constant along each trajectory and inference reduces to
$\mathrm{Round}(X_0 + \sigma_0\varepsilon + v_\theta(\mathbf{c}))$: a
deterministic score matrix perturbed by isotropic noise and rounded, the
same structure as Gumbel-Sinkhorn with Gaussian rather than Gumbel noise.
Its optimal constant field is the mean target direction, so it can
produce several permutations only through noise around one mean
prediction. We use it as an ablation.

\subsection{Deterministic and stochastic Sinkhorn}
For a fixed score matrix $S$, the entropic Sinkhorn operator is a
deterministic function of $S$ \citep{sinkhorn1964relationship}, so any
deterministic rounding of it is a point mass on one permutation per input.
Gumbel-Sinkhorn adds i.i.d.\ Gumbel noise and defines a distribution over
permutations whose exact form depends on the decoder: rounding the
Sinkhorn output, $\mathrm{Round}(\mathrm{Sinkhorn}((S+G)/\tau))$, or
rounding the perturbed scores, $\mathrm{Round}(S+G)$ (the Gumbel-Matching
distribution of $S$). The two coincide only after taking logarithms,
$\mathrm{Round}(\log\mathrm{Sinkhorn}((S+G)/\tau))=\mathrm{Round}(S+G)$; they
differ in general because the assignment objective is not invariant to the
entrywise logarithm (Appendix~\ref{app:background} gives a $3\times3$
example). We report both decoders for the GS baselines.

\section{When a Single Perturbed Score Suffices}
\label{app:structure}

Gumbel-Sinkhorn and Plackett--Luce sample a permutation by perturbing one
score object per input with i.i.d.\ noise and taking an argmax. This
appendix gives one positive and one negative result about such samplers,
which delimit when a flow can add expressive power over such samplers:
Proposition~\ref{prop:orbit-marginal} shows that symmetric scores with
i.i.d.\ noise sample an orbit of identical items uniformly, and
Proposition~\ref{prop:pl-unimodal} shows that a Plackett--Luce model has a
single local mode.

\subsection{Symmetries are free}
Let $H$ be a group of permutations of the $n$ positions, represented by
permutation matrices acting by right multiplication, and let
$\mathcal{O}=\{P^1Q:Q\in H\}$ be the orbit of a permutation $P^1$. For
jigsaw puzzles with identical patches, $H$ is the group that permutes the
slots within each class of identical patches, and $\mathcal{O}$ is the set
of perfect reassemblies.

\begin{proposition}[Symmetric scores sample orbits uniformly]
\label{prop:orbit-marginal}
Let $S\in\R^{n\times n}$ satisfy $SQ=S$ for every $Q\in H$, let $G$ have
i.i.d.\ entries with a continuous distribution, and let
$\hat P=\mathrm{Round}(S+G)$. Then, conditional on
$\hat P\in\mathcal{O}$, $\hat P$ is uniform on $\mathcal{O}$.
\end{proposition}
\begin{proof}
For $Q\in H$, $GQ$ has the same law as $G$ (its entries are a permutation
of i.i.d.\ entries), so $(S+G)Q=S+GQ$ has the same law as $S+G$. The
assignment problem is equivariant, $\mathrm{Round}(XQ)=\mathrm{Round}(X)Q$,
and its solution is almost surely unique for a continuous $G$; hence
$\hat PQ$ has the same law as $\hat P$. The event
$\{\hat P\in\mathcal{O}\}$ is invariant under right multiplication by
$H$, so the conditional law of $\hat P$ on $\mathcal{O}$ is $H$-invariant.
$H$ acts transitively on $\mathcal{O}$, so the only invariant
distribution is the uniform one.
\end{proof}

Because Sinkhorn scaling is equivariant,
$\mathrm{Sinkhorn}(XQ)=\mathrm{Sinkhorn}(X)Q$, the same argument applies to
the soft decoder $\mathrm{Round}(\mathrm{Sinkhorn}((S+G)/\tau))$.
A score matrix that depends on each slot only through its content has
the required symmetry, as does the orbit mean $\bar P(\mathcal{O})$: once
the unique items are placed, the multiplicity of an orbit is supplied by
exchangeable noise.

\subsection{Plackett--Luce is unimodal}
\begin{proposition}[Plackett--Luce has a single local mode]
\label{prop:pl-unimodal}
Let $s\in\R^n$ have distinct entries and let $\pi$ be any order that is
not sorted by decreasing $s$. Then $\pi$ has two adjacent items
$\pi_k,\pi_{k+1}$ with $s_{\pi_k}<s_{\pi_{k+1}}$, and swapping them strictly
increases $p_{\mathrm{PL}}(\cdot\mid s)$. Consequently the sorted order is
the unique local maximum of $p_{\mathrm{PL}}$ with respect to adjacent
transpositions.
\end{proposition}
\begin{proof}
Write $a=e^{s_{\pi_k}}$, $b=e^{s_{\pi_{k+1}}}$, and $r$ for the sum of
$e^{s_{\pi_l}}$ over $l>k+1$. The factors of $p_{\mathrm{PL}}$ for
positions other than $k,k+1$ are unchanged by the swap, and
\[
\frac{p(\ldots,\pi_k,\pi_{k+1},\ldots)}{p(\ldots,\pi_{k+1},\pi_k,\ldots)}
=\frac{\frac{a}{a+b+r}\cdot\frac{b}{b+r}}{\frac{b}{a+b+r}\cdot\frac{a}{a+r}}
=\frac{a+r}{b+r}<1
\]
when $a<b$.
\end{proof}

A mixture of two Mallows models with reversed centres and dispersion
$\theta$ large enough has two local maxima, at the two centres, so no
single Plackett--Luce model can represent it; the best Plackett--Luce fit
places substantial mass on the rankings between the two centres. The same
argument does not apply verbatim to Gumbel-matching distributions over
assignments, but GS shares the structure of one score object plus
i.i.d.\ noise, and Section~\ref{sec:mallows} finds that it fails in the
same way.

\section{Visual Sorting: Details and Full Results}
\label{app:mnist}

\begin{figure}[htbp]
  \centering
  \includegraphics[width=0.85\linewidth]{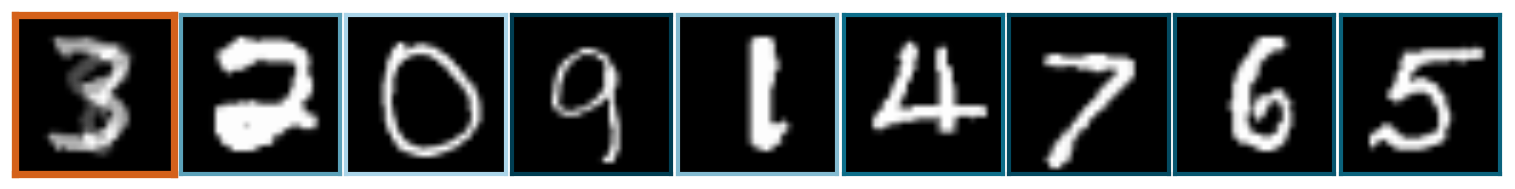}
  \caption{Visual sorting with ambiguity. The row is one input of nine
  digits. The highlighted image is a blend of two digits and admits two
  valid sorted orders $P_a$ and $P_b$; the blend weight $\alpha$ defines
  their target probabilities.}
  \label{fig:digit-sort-task}
\end{figure}

\subsection{Data}
Figure~\ref{fig:digit-sort-task} shows an example input. Sequences of $N=9$ digit images are drawn from MNIST. A clean sequence
consists of nine images drawn without replacement and placed in random
order; its unique target sorts them by digit value. In an ambiguous
sequence one position holds a blend
$\mathbf{x}_{\text{blend}} = \alpha\,\mathbf{x}_a + (1-\alpha)\,\mathbf{x}_b$
with $\alpha \sim \mathrm{Beta}(2,2)$ clipped to $[0.2, 0.8]$, and the pair
of digit values $(d_a,d_b)$ is drawn from
$\{(0,4),(1,5),(2,6),(3,7),(1,6),(2,7),(0,5),(3,8)\}$, whose values differ
by at least three, so that $d_a$ and $d_b$ occupy different sorted
positions and the sequence has exactly two valid orders $P_a,P_b$. The training set has 100{,}000 sequences, half of
them ambiguous; it is built once with seed 42 and shared by all runs.
The test sets contain 2{,}000 clean and 2{,}000 ambiguous sequences made
only from MNIST test images.

PermFlow is trained with set-valued supervision $\{P_a,P_b\}$ under
either the nearest-target coupling, which does not use $\alpha$, or the
calibrated coupling, which sets the cell of $P_a$ to mass $\alpha$. GS is
trained on a single valid order per example: always $P_a$, always $P_b$,
$P_a$ or $P_b$ with probability $\tfrac12$ each, redrawn at every step
(none of which uses $\alpha$), or $P_a$ with probability $\alpha$ and $P_b$
otherwise, redrawn at every step (which uses $\alpha$, like the calibrated
coupling). We also train the state-independent variant of PermFlow
(Appendix~\ref{app:stateindep}) with the nearest coupling. All models have
$0.83$--$0.86$M parameters.

\subsection{Metrics}
On clean sequences we report the exact-match accuracy of one sample. On
ambiguous sequences we draw $K$ samples and report Coverage@$K$, the
fraction of instances on which both $P_a$ and $P_b$ appear among the
samples; Any-correct@$K$, the fraction on which at least one sample is
valid; and the calibration error $\mathbb{E}\,|\hat f_a-\alpha|$, where
$\hat f_a$ is the fraction of the $K$ samples equal to $P_a$. Three
reference values anchor the calibration error under our distribution of
$\alpha$. A perfect sampler has error
$\mathbb{E}|\mathrm{Bin}(K,\alpha)/K-\alpha|$ from finite $K$: $0.163$,
$0.115$, $0.082$ and $0.037$ at $K=5,10,20,100$. A sampler that splits
$50{:}50$ regardless of $\alpha$ tends to $\mathbb{E}|\alpha-\tfrac12|=0.173$,
and a deterministic predictor of the more likely order has
$\mathbb{E}\min(\alpha,1-\alpha)=0.327$.

\subsection{Models}
Each image is encoded by a small CNN (convolutions with 16, 32 and 64
filters of size $3\times3$, BatchNorm, ReLU, $2\times2$ max-pooling, adaptive
average pooling, and a linear layer with LayerNorm to $d_{\text{enc}}=64$;
27{,}808 parameters). The nine embeddings are projected to width 128,
given learnable slot embeddings, and processed by a 4-layer pre-norm
transformer (4 heads, feed-forward width 512, dropout 0.1, GELU).

In PermFlow, the state $X_t$ is flattened, projected to width 128 by a linear layer and
LayerNorm, and added to an embedding of $t$ (a two-layer MLP with SiLU);
the result is prepended as an extra token and removed after the
transformer. A bilinear head $f_{ij}=\langle h_i,r_j\rangle/\sqrt{128}$
with learnable rank embeddings $r_j$ gives the velocity $\C(f)$.
859{,}552 parameters. The state-independent variant has no state token and
no time input (832{,}032 parameters); its velocity is constant along a
trajectory and it is trained with target $P^\star-X_0$, the mean of
$P^\star-\tilde X_0$.

The nearest coupling is $P^\star=\argmin_{P\in\{P_a,P_b\}}\|\tilde X_0-P\|_F$.
The calibrated coupling selects $P_a$ iff $\inner{\varepsilon}{d}\ge q(\alpha)$ with
$d=(P_a-P_b)/\|P_a-P_b\|_F$ and
$q(\alpha)=2F^{-1}_{\mathrm{Beta}(c,c)}(1-\alpha)-1$, $c=((N-1)^2-1)/2$. On
$2\times10^4$ draws the empirical frequency of $P_a$ is $0.201$, $0.509$
and $0.809$ for $\alpha=0.2,0.5,0.8$. The random coupling (an ablation) selects $P_a$ with
probability $\tfrac12$ independently of $\varepsilon$. Clean sequences
have $P_a=P_b$.

Gumbel-Sinkhorn uses the same encoder and transformer without the state
token, and reads the same bilinear head as a score matrix $S$ (832{,}032 parameters). The loss
is $\|\mathrm{Sinkhorn}((S+G)/\tau)-\bar P\|_F^2$ with $\tau=1$ and 20
log-domain Sinkhorn iterations. The target $\bar P$ is a one-hot valid order on ambiguous sequences:
always $P_a$, always $P_b$, either with probability $\tfrac12$, or $P_a$
with probability $\alpha$, the last two redrawn at every step; clean
sequences have $\bar P=P_a=P_b$. Unless stated
otherwise, samples use the soft decoder of the original Gumbel-Sinkhorn,
$\mathrm{Round}(\mathrm{Sinkhorn}((S+G)/\tau))$ with fresh $G$;
Table~\ref{tab:mnist_decoders} compares the Gumbel-Matching decoder
$\mathrm{Round}(S+G)$.

Training uses AdamW (learning rate $3\times10^{-4}$, weight decay $10^{-4}$), cosine
decay to $10^{-5}$, 300 epochs, batch size 256, mixed precision, gradient
clipping at norm 1.0, three seeds (42, 0, 1) per configuration, final-epoch
model, single GPU. PermFlow uses $\sigma_0=1$ and $S=10$ Euler steps.

\subsection{Full results}
Table~\ref{tab:mnist_main} reports means and standard deviations over
three seeds on the 2{,}000 clean and 2{,}000 ambiguous test sequences,
with GS decoded by rounding the Sinkhorn output; the reference values of
the calibration error are given above ($0.115$ and $0.037$ for a perfect
sampler at $K=10$ and $100$, $0.173$ for a $50{:}50$ sampler, $0.327$ for
always choosing the more likely order).
Table~\ref{tab:mnist_full} reports all sampling budgets over three seeds,
including the two one-seed ablations (random coupling, no projection),
which are indistinguishable from the nearest coupling; clean accuracy is
computed from one sample in each evaluation run.
Table~\ref{tab:mnist_steps} varies the number of Euler steps for the
trained PermFlow models (three seeds). Coverage and calibration error are
essentially independent of $S$, as predicted by Theorem~\ref{thm:unbiased}.

\begin{table}[htbp]
\centering
\scriptsize
\caption{Visual sorting at all sampling budgets $K$.}
\label{tab:mnist_full}
\begin{tabular}{llcccc}
\toprule
Method & $K$ & Clean acc.\ (\%) & Coverage (\%) & Any-correct (\%) & Cal.\ error \\
\midrule
\multirow{7}{*}{GS, only $P_a$} & 5 & 98.6{\scriptsize$\pm$0.1} & 0.0{\scriptsize$\pm$0.0} & 99.9{\scriptsize$\pm$0.0} & 0.495{\scriptsize$\pm$0.000} \\
 & 10 & 98.6{\scriptsize$\pm$0.1} & 0.0{\scriptsize$\pm$0.0} & 99.9{\scriptsize$\pm$0.1} & 0.495{\scriptsize$\pm$0.001} \\
 & 20 & 98.6{\scriptsize$\pm$0.2} & 0.0{\scriptsize$\pm$0.0} & 99.9{\scriptsize$\pm$0.1} & 0.495{\scriptsize$\pm$0.001} \\
 & 40 & 98.6{\scriptsize$\pm$0.1} & 0.0{\scriptsize$\pm$0.0} & 99.9{\scriptsize$\pm$0.0} & 0.495{\scriptsize$\pm$0.001} \\
 & 60 & 98.5{\scriptsize$\pm$0.1} & 0.0{\scriptsize$\pm$0.0} & 100.0{\scriptsize$\pm$0.0} & 0.495{\scriptsize$\pm$0.000} \\
 & 80 & 98.6{\scriptsize$\pm$0.1} & 0.0{\scriptsize$\pm$0.0} & 99.9{\scriptsize$\pm$0.0} & 0.495{\scriptsize$\pm$0.001} \\
 & 100 & 98.6{\scriptsize$\pm$0.2} & 0.0{\scriptsize$\pm$0.0} & 100.0{\scriptsize$\pm$0.0} & 0.495{\scriptsize$\pm$0.000} \\
\midrule
\multirow{7}{*}{GS, only $P_b$} & 5 & 98.4{\scriptsize$\pm$0.3} & 0.0{\scriptsize$\pm$0.0} & 99.8{\scriptsize$\pm$0.0} & 0.504{\scriptsize$\pm$0.000} \\
 & 10 & 98.5{\scriptsize$\pm$0.4} & 0.0{\scriptsize$\pm$0.0} & 99.8{\scriptsize$\pm$0.1} & 0.504{\scriptsize$\pm$0.000} \\
 & 20 & 98.4{\scriptsize$\pm$0.4} & 0.0{\scriptsize$\pm$0.0} & 99.8{\scriptsize$\pm$0.1} & 0.504{\scriptsize$\pm$0.000} \\
 & 40 & 98.4{\scriptsize$\pm$0.4} & 0.0{\scriptsize$\pm$0.0} & 99.8{\scriptsize$\pm$0.1} & 0.504{\scriptsize$\pm$0.000} \\
 & 60 & 98.5{\scriptsize$\pm$0.3} & 0.0{\scriptsize$\pm$0.0} & 99.9{\scriptsize$\pm$0.1} & 0.504{\scriptsize$\pm$0.000} \\
 & 80 & 98.4{\scriptsize$\pm$0.4} & 0.0{\scriptsize$\pm$0.0} & 99.9{\scriptsize$\pm$0.1} & 0.504{\scriptsize$\pm$0.000} \\
 & 100 & 98.4{\scriptsize$\pm$0.4} & 0.0{\scriptsize$\pm$0.0} & 99.9{\scriptsize$\pm$0.1} & 0.504{\scriptsize$\pm$0.000} \\
\midrule
\multirow{7}{*}{GS, random $P_a$ or $P_b$} & 5 & 98.6{\scriptsize$\pm$0.2} & 93.1{\scriptsize$\pm$0.5} & 99.8{\scriptsize$\pm$0.1} & 0.243{\scriptsize$\pm$0.004} \\
 & 10 & 98.6{\scriptsize$\pm$0.1} & 99.5{\scriptsize$\pm$0.1} & 99.8{\scriptsize$\pm$0.0} & 0.208{\scriptsize$\pm$0.001} \\
 & 20 & 98.7{\scriptsize$\pm$0.2} & 99.7{\scriptsize$\pm$0.1} & 99.9{\scriptsize$\pm$0.0} & 0.193{\scriptsize$\pm$0.002} \\
 & 40 & 98.7{\scriptsize$\pm$0.2} & 99.7{\scriptsize$\pm$0.0} & 99.9{\scriptsize$\pm$0.1} & 0.183{\scriptsize$\pm$0.001} \\
 & 60 & 98.7{\scriptsize$\pm$0.2} & 99.8{\scriptsize$\pm$0.1} & 99.9{\scriptsize$\pm$0.0} & 0.178{\scriptsize$\pm$0.001} \\
 & 80 & 98.7{\scriptsize$\pm$0.1} & 99.8{\scriptsize$\pm$0.1} & 99.9{\scriptsize$\pm$0.0} & 0.178{\scriptsize$\pm$0.001} \\
 & 100 & 98.7{\scriptsize$\pm$0.2} & 99.8{\scriptsize$\pm$0.1} & 100.0{\scriptsize$\pm$0.0} & 0.177{\scriptsize$\pm$0.002} \\
\midrule
\multirow{7}{*}{GS, $P_a$ w.p.\ $\alpha$} & 5 & 98.6{\scriptsize$\pm$0.1} & 38.3{\scriptsize$\pm$0.6} & 99.9{\scriptsize$\pm$0.0} & 0.228{\scriptsize$\pm$0.001} \\
 & 10 & 98.6{\scriptsize$\pm$0.2} & 48.6{\scriptsize$\pm$1.4} & 99.9{\scriptsize$\pm$0.0} & 0.212{\scriptsize$\pm$0.003} \\
 & 20 & 98.6{\scriptsize$\pm$0.1} & 56.1{\scriptsize$\pm$0.8} & 99.9{\scriptsize$\pm$0.0} & 0.204{\scriptsize$\pm$0.003} \\
 & 40 & 98.6{\scriptsize$\pm$0.1} & 62.1{\scriptsize$\pm$1.3} & 99.9{\scriptsize$\pm$0.0} & 0.200{\scriptsize$\pm$0.003} \\
 & 60 & 98.6{\scriptsize$\pm$0.1} & 65.4{\scriptsize$\pm$1.3} & 99.9{\scriptsize$\pm$0.0} & 0.198{\scriptsize$\pm$0.001} \\
 & 80 & 98.7{\scriptsize$\pm$0.1} & 67.5{\scriptsize$\pm$1.4} & 99.9{\scriptsize$\pm$0.0} & 0.196{\scriptsize$\pm$0.002} \\
 & 100 & 98.7{\scriptsize$\pm$0.1} & 69.0{\scriptsize$\pm$2.0} & 99.9{\scriptsize$\pm$0.0} & 0.196{\scriptsize$\pm$0.002} \\
\midrule
\multirow{7}{*}{PermFlow, state-independent} & 5 & 98.4{\scriptsize$\pm$0.3} & 91.2{\scriptsize$\pm$1.2} & 99.3{\scriptsize$\pm$0.1} & 0.244{\scriptsize$\pm$0.002} \\
 & 10 & 98.3{\scriptsize$\pm$0.2} & 98.6{\scriptsize$\pm$0.2} & 99.3{\scriptsize$\pm$0.1} & 0.213{\scriptsize$\pm$0.003} \\
 & 20 & 98.4{\scriptsize$\pm$0.3} & 99.1{\scriptsize$\pm$0.1} & 99.4{\scriptsize$\pm$0.1} & 0.197{\scriptsize$\pm$0.002} \\
 & 40 & 98.4{\scriptsize$\pm$0.3} & 99.2{\scriptsize$\pm$0.2} & 99.5{\scriptsize$\pm$0.1} & 0.187{\scriptsize$\pm$0.001} \\
 & 60 & 98.4{\scriptsize$\pm$0.3} & 99.3{\scriptsize$\pm$0.1} & 99.5{\scriptsize$\pm$0.0} & 0.183{\scriptsize$\pm$0.001} \\
 & 80 & 98.3{\scriptsize$\pm$0.3} & 99.3{\scriptsize$\pm$0.1} & 99.5{\scriptsize$\pm$0.1} & 0.181{\scriptsize$\pm$0.001} \\
 & 100 & 98.4{\scriptsize$\pm$0.3} & 99.4{\scriptsize$\pm$0.2} & 99.5{\scriptsize$\pm$0.1} & 0.180{\scriptsize$\pm$0.002} \\
\midrule
\multirow{7}{*}{PermFlow, nearest coupling} & 5 & 97.5{\scriptsize$\pm$0.4} & 92.4{\scriptsize$\pm$0.3} & 99.4{\scriptsize$\pm$0.1} & 0.240{\scriptsize$\pm$0.002} \\
 & 10 & 97.8{\scriptsize$\pm$0.3} & 98.8{\scriptsize$\pm$0.2} & 99.4{\scriptsize$\pm$0.2} & 0.212{\scriptsize$\pm$0.002} \\
 & 20 & 97.7{\scriptsize$\pm$0.3} & 99.3{\scriptsize$\pm$0.1} & 99.5{\scriptsize$\pm$0.0} & 0.190{\scriptsize$\pm$0.003} \\
 & 40 & 97.6{\scriptsize$\pm$0.4} & 99.4{\scriptsize$\pm$0.2} & 99.6{\scriptsize$\pm$0.0} & 0.182{\scriptsize$\pm$0.000} \\
 & 60 & 97.8{\scriptsize$\pm$0.4} & 99.4{\scriptsize$\pm$0.1} & 99.6{\scriptsize$\pm$0.1} & 0.179{\scriptsize$\pm$0.001} \\
 & 80 & 97.5{\scriptsize$\pm$0.3} & 99.5{\scriptsize$\pm$0.1} & 99.6{\scriptsize$\pm$0.1} & 0.178{\scriptsize$\pm$0.001} \\
 & 100 & 97.7{\scriptsize$\pm$0.3} & 99.5{\scriptsize$\pm$0.2} & 99.6{\scriptsize$\pm$0.0} & 0.177{\scriptsize$\pm$0.002} \\
\midrule
\multirow{7}{*}{PermFlow, calibrated coupling} & 5 & 97.8{\scriptsize$\pm$0.2} & 81.0{\scriptsize$\pm$0.0} & 99.3{\scriptsize$\pm$0.1} & 0.172{\scriptsize$\pm$0.000} \\
 & 10 & 97.9{\scriptsize$\pm$0.1} & 94.6{\scriptsize$\pm$0.2} & 99.5{\scriptsize$\pm$0.0} & 0.129{\scriptsize$\pm$0.002} \\
 & 20 & 97.8{\scriptsize$\pm$0.1} & 98.5{\scriptsize$\pm$0.2} & 99.5{\scriptsize$\pm$0.0} & 0.097{\scriptsize$\pm$0.002} \\
 & 40 & 97.8{\scriptsize$\pm$0.2} & 99.1{\scriptsize$\pm$0.1} & 99.5{\scriptsize$\pm$0.0} & 0.078{\scriptsize$\pm$0.001} \\
 & 60 & 97.8{\scriptsize$\pm$0.2} & 99.3{\scriptsize$\pm$0.1} & 99.5{\scriptsize$\pm$0.0} & 0.071{\scriptsize$\pm$0.000} \\
 & 80 & 97.8{\scriptsize$\pm$0.1} & 99.4{\scriptsize$\pm$0.1} & 99.6{\scriptsize$\pm$0.0} & 0.065{\scriptsize$\pm$0.001} \\
 & 100 & 97.8{\scriptsize$\pm$0.0} & 99.4{\scriptsize$\pm$0.1} & 99.6{\scriptsize$\pm$0.0} & 0.063{\scriptsize$\pm$0.000} \\
\midrule
\multirow{7}{*}{Ablation: random coupling} & 5 & 97.7 & 92.0 & 99.3 & 0.246 \\
 & 10 & 97.8 & 98.6 & 99.2 & 0.209 \\
 & 20 & 97.8 & 99.0 & 99.3 & 0.194 \\
 & 40 & 97.9 & 99.2 & 99.4 & 0.181 \\
 & 60 & 97.5 & 99.2 & 99.4 & 0.179 \\
 & 80 & 97.5 & 99.2 & 99.4 & 0.177 \\
 & 100 & 97.5 & 99.2 & 99.4 & 0.178 \\
\midrule
\multirow{7}{*}{Ablation: no projection} & 5 & 97.8 & 91.5 & 99.4 & 0.240 \\
 & 10 & 97.8 & 98.8 & 99.6 & 0.211 \\
 & 20 & 98.0 & 99.3 & 99.6 & 0.191 \\
 & 40 & 97.8 & 99.4 & 99.7 & 0.183 \\
 & 60 & 97.8 & 99.4 & 99.7 & 0.181 \\
 & 80 & 97.8 & 99.5 & 99.7 & 0.178 \\
 & 100 & 97.8 & 99.4 & 99.7 & 0.175 \\
\bottomrule
\end{tabular}
\end{table}

\begin{table}[htbp]
\centering
\scriptsize
\caption{Visual sorting: number of Euler steps $S$. Coverage and calibration
error are essentially independent of $S$.}
\label{tab:mnist_steps}
\setlength{\tabcolsep}{3pt}
\begin{tabular}{ccccccccc}
\toprule
 & \multicolumn{4}{c}{Nearest coupling} & \multicolumn{4}{c}{Calibrated coupling} \\
\cmidrule(lr){2-5}\cmidrule(lr){6-9}
$S$ & Cov@10 & Cal@10 & Cov@100 & Cal@100 & Cov@10 & Cal@10 & Cov@100 & Cal@100 \\
\midrule
1 & 98.6{\scriptsize$\pm$0.1} & 0.213{\scriptsize$\pm$0.001} & 99.4{\scriptsize$\pm$0.1} & 0.178{\scriptsize$\pm$0.001} & 95.1{\scriptsize$\pm$0.3} & 0.128{\scriptsize$\pm$0.001} & 99.2{\scriptsize$\pm$0.1} & 0.064{\scriptsize$\pm$0.001} \\
2 & 98.6{\scriptsize$\pm$0.0} & 0.210{\scriptsize$\pm$0.002} & 99.4{\scriptsize$\pm$0.2} & 0.174{\scriptsize$\pm$0.001} & 94.2{\scriptsize$\pm$0.1} & 0.127{\scriptsize$\pm$0.002} & 99.2{\scriptsize$\pm$0.2} & 0.065{\scriptsize$\pm$0.001} \\
5 & 98.6{\scriptsize$\pm$0.0} & 0.211{\scriptsize$\pm$0.001} & 99.5{\scriptsize$\pm$0.2} & 0.175{\scriptsize$\pm$0.001} & 94.7{\scriptsize$\pm$0.4} & 0.127{\scriptsize$\pm$0.002} & 99.3{\scriptsize$\pm$0.2} & 0.063{\scriptsize$\pm$0.001} \\
10 & 98.6{\scriptsize$\pm$0.0} & 0.211{\scriptsize$\pm$0.001} & 99.5{\scriptsize$\pm$0.2} & 0.176{\scriptsize$\pm$0.001} & 95.1{\scriptsize$\pm$0.5} & 0.127{\scriptsize$\pm$0.002} & 99.4{\scriptsize$\pm$0.2} & 0.063{\scriptsize$\pm$0.001} \\
20 & 98.6{\scriptsize$\pm$0.0} & 0.211{\scriptsize$\pm$0.001} & 99.5{\scriptsize$\pm$0.2} & 0.176{\scriptsize$\pm$0.001} & 95.2{\scriptsize$\pm$0.5} & 0.127{\scriptsize$\pm$0.002} & 99.4{\scriptsize$\pm$0.2} & 0.063{\scriptsize$\pm$0.001} \\
50 & 98.7{\scriptsize$\pm$0.0} & 0.212{\scriptsize$\pm$0.001} & 99.5{\scriptsize$\pm$0.2} & 0.176{\scriptsize$\pm$0.001} & 95.3{\scriptsize$\pm$0.5} & 0.127{\scriptsize$\pm$0.002} & 99.4{\scriptsize$\pm$0.2} & 0.063{\scriptsize$\pm$0.001} \\
100 & 98.7{\scriptsize$\pm$0.0} & 0.212{\scriptsize$\pm$0.001} & 99.5{\scriptsize$\pm$0.2} & 0.176{\scriptsize$\pm$0.001} & 95.3{\scriptsize$\pm$0.5} & 0.127{\scriptsize$\pm$0.002} & 99.4{\scriptsize$\pm$0.2} & 0.063{\scriptsize$\pm$0.001} \\
\bottomrule
\end{tabular}
\end{table}

Table~\ref{tab:mnist_decoders} repeats the GS evaluation on the same
trained models (three seeds) with the Gumbel-Matching decoder
$\mathrm{Round}(S+G)$ instead of rounding the Sinkhorn output
$\mathrm{Round}(\mathrm{Sinkhorn}((S+G)/\tau))$ used in
Table~\ref{tab:mnist_main}. The two decoders are not equivalent
(Appendix~\ref{app:background}), but the conclusions do not change: with
either decoder, GS trained on a fixed order has calibration error $0.49$--$0.50$
at $K=100$ and GS trained on a randomly drawn order $0.18$--$0.20$, against
$0.063$ for PermFlow with the calibrated coupling.

\begin{table}[htbp]
\centering
\scriptsize
\caption{Visual sorting: GS with the soft and the Gumbel-Matching decoder. The
decoder does not change the conclusions.}
\label{tab:mnist_decoders}
\setlength{\tabcolsep}{3pt}
\begin{tabular}{llccccc}
\toprule
 & & Clean (\%) & \multicolumn{2}{c}{Coverage (\%)} & \multicolumn{2}{c}{Calibration error} \\
\cmidrule(lr){4-5}\cmidrule(lr){6-7}
Target & Decoder & $K{=}1$ & $K{=}10$ & $K{=}100$ & $K{=}10$ & $K{=}100$ \\
\midrule
only $P_a$ & soft & 98.6{\scriptsize$\pm$0.1} & 0.0{\scriptsize$\pm$0.0} & 0.0{\scriptsize$\pm$0.0} & 0.495{\scriptsize$\pm$0.001} & 0.495{\scriptsize$\pm$0.000} \\
 & Gumbel-Matching & 98.5{\scriptsize$\pm$0.2} & 0.0{\scriptsize$\pm$0.0} & 0.0{\scriptsize$\pm$0.0} & 0.495{\scriptsize$\pm$0.000} & 0.494{\scriptsize$\pm$0.000} \\
\midrule
only $P_b$ & soft & 98.5{\scriptsize$\pm$0.4} & 0.0{\scriptsize$\pm$0.0} & 0.0{\scriptsize$\pm$0.0} & 0.504{\scriptsize$\pm$0.000} & 0.504{\scriptsize$\pm$0.000} \\
 & Gumbel-Matching & 98.2{\scriptsize$\pm$0.5} & 0.0{\scriptsize$\pm$0.0} & 0.0{\scriptsize$\pm$0.0} & 0.504{\scriptsize$\pm$0.000} & 0.504{\scriptsize$\pm$0.000} \\
\midrule
random & soft & 98.6{\scriptsize$\pm$0.1} & 99.5{\scriptsize$\pm$0.1} & 99.8{\scriptsize$\pm$0.1} & 0.208{\scriptsize$\pm$0.001} & 0.177{\scriptsize$\pm$0.002} \\
 & Gumbel-Matching & 98.6{\scriptsize$\pm$0.2} & 97.1{\scriptsize$\pm$2.3} & 99.7{\scriptsize$\pm$0.0} & 0.219{\scriptsize$\pm$0.010} & 0.191{\scriptsize$\pm$0.011} \\
\midrule
$P_a$ w.p.\ $\alpha$ & soft & 98.6{\scriptsize$\pm$0.2} & 48.6{\scriptsize$\pm$1.4} & 69.0{\scriptsize$\pm$2.0} & 0.212{\scriptsize$\pm$0.003} & 0.196{\scriptsize$\pm$0.002} \\
 & Gumbel-Matching & 98.6{\scriptsize$\pm$0.2} & 46.2{\scriptsize$\pm$0.9} & 68.0{\scriptsize$\pm$1.8} & 0.212{\scriptsize$\pm$0.010} & 0.195{\scriptsize$\pm$0.012} \\
\bottomrule
\end{tabular}
\end{table}

\section{Jigsaw Puzzles: Details and Additional Results}
\label{app:jigsaw}

\subsection{Data and metrics}
We follow the benchmark of \citet{zhang2025symmetricdiffusers}: MNIST
provides 60{,}000 training and 10{,}000 test images; each image is cut
into a $g\times g$ grid with patch side $\lfloor 28/g\rfloor$ in raster
order, discarding the right and bottom remainder exactly as their
\texttt{chunk\_image()} does (so $g=6$ uses the top-left $24\times24$
pixels); and each image is shuffled by one fixed uniformly random
permutation. Their repository distributes the noised data as a
pre-chunked artifact without generation code or the value of the noise
level, so we regenerate Noisy MNIST as MNIST plus i.i.d.\ Gaussian pixel
noise with $\sigma=0.01$, kept in the native 8-bit domain. The data are
therefore not byte-identical to theirs; our rerun of their released model
reproduces their published numbers at $4\times4$ and $5\times5$
(see below).

Accuracy scores an image as correct only if all $n$ patches are correct;
Correct is the fraction of correctly placed patches; Kendall-$\tau$ is
computed between predicted and true orderings; RMSE and MAE compare the
reconstructed and original images. Single-sample metrics use the full
test split. The $K$-sample metrics of Table~\ref{tab:orbit} use a fixed
2{,}000-image prefix of the test split.

In Table~\ref{tab:jigsaw}, the baseline numbers are taken from Table~1 of
\citet{zhang2025symmetricdiffusers} (averages over five runs;
SymmetricDiffusers decodes with beam search, with PL beam size 200 and
time beam size 20). PermFlow decodes a single sample ($K{=}1$) with $S=10$
Euler steps and is averaged over three seeds; the column ``PermFlow,
$S{=}1$'' evaluates the seed-42 models with one Euler step and is excluded
from the bold marking.

The $\sigma=0$ arm is our extension of the benchmark rather than part of
it, and the published baselines were not designed for it. In
Table~\ref{tab:orbit}, Valid is the fraction of test
images reassembled exactly, into any member of the orbit, by one sample.
Dist.\ is the number of distinct valid reassemblies among $K$ samples,
normalised by $\min(K,|\mathcal{O}|)$. TV is the total-variation distance,
averaged over patches, between the empirical distribution of a patch's
slot over the $K$ samples and the uniform distribution over its admissible
slots; because $K$ is finite, a uniform sampler has TV $0.098$ at $K=100$
on $4\times4$. Corr.(c) is the fraction of patches placed in an admissible
slot. PermFlow is averaged over three seeds; SymmetricDiffusers, sampled
from its reverse process, and the ablations use one seed.

PermFlow uses set-valued supervision: the target set is the orbit (a
single permutation when $\sigma=0.01$), and the coupling is the nearest
orbit member, computed exactly by one small assignment problem per class
of identical patches, at a cost independent of $|\mathcal{O}|$. By
Corollary~\ref{cor:coupling}(c), the optimal sampler under this coupling is
uniform on the orbit. The encoder and transformer (7 layers, 8 heads,
width 128) follow the SymmetricDiffusers backbone; PermFlow has 1.53M,
1.53M and 1.61M parameters at the three grid sizes, against 1.52M, 1.47M
and 1.47M for SymmetricDiffusers. Hyperparameters are copied from the
sorting task without tuning, and every number is a single sample
($K{=}1$) from the final-epoch model.

\subsection{Models}
The per-patch CNN follows SymmetricDiffusers layer for layer
($\mathrm{conv}(32,3{\times}3)\to\mathrm{BatchNorm}\to\mathrm{ReLU}\to
\mathrm{MaxPool}(2)\to\mathrm{Linear}\to\mathrm{ReLU}\to\mathrm{Linear}$ to
width 128), as do the transformer hyperparameters (7 layers, 8 heads,
width 128, feed-forward width 512, dropout 0.1, GELU, pre-norm), with
learnable patch-position embeddings.

In PermFlow, the state $X_t$ is flattened, projected to width 128 with a linear layer
and LayerNorm, added to a time embedding (two-layer MLP with SiLU), and
prepended as a token that is dropped after the transformer. A bilinear head
against learnable slot embeddings gives $f\in\R^{n\times n}$, and $\C(f)$ is
the velocity. Parameters: 1.53M, 1.53M and 1.61M at $g=4,5,6$; the
state projection grows as $n^2$. The state-independent ablation drops the
state token and time input (1.48M at $g=4$).

For the orbit coupling, patches are grouped by exact byte equality of their content, which
determines the orbit $\mathcal{O}$. The coupling
$P^\star=\argmin_{P\in\mathcal{O}}\|\tilde X_0-P\|_F$ decomposes over the
classes: unique patches are fixed, and each class of $m_b$ identical
patches contributes an independent $m_b\times m_b$ assignment problem on
the corresponding submatrix of $\tilde X_0$. The cost is
$O(\sum_b m_b^3)$, independent of $|\mathcal{O}|$. The noise and the
coupling are redrawn at every step. At $\sigma=0.01$ every patch is unique
and the coupling returns the single label. The single-label ablation
instead always uses the one labelled permutation.

To validate the regenerated data we also ran the released
SymmetricDiffusers implementation on it for 120 epochs with its published configuration and
beam-search decoder. It reproduces the published numbers
(Kendall-$\tau$ $0.482$ vs.\ $0.486$ at $4\times4$ and $0.293$ vs.\
$0.285$ at $5\times5$; Correct $58.3\%$ vs.\ $58.5\%$ and $38.9\%$ vs.\
$37.9\%$). For the orbit arm we sample from its reverse process instead,
since beam search returns a single output.

Training uses AdamW (learning rate $3\times10^{-4}$, weight decay $10^{-4}$), cosine
decay to $10^{-5}$, 300 epochs, batch size 256, mixed precision, gradient
clipping at norm 1.0, three seeds (42, 0, 1) per configuration, final-epoch
model. An earlier version of this work selected the
checkpoint with the best validity on the test prefix; we no longer do so.
These hyperparameters are copied from the sorting task; none was tuned
for this task. PermFlow uses $\sigma_0=1$ and, unless stated otherwise,
$S=10$ Euler steps.

\subsection{Additional results}
In the published setting, PermFlow raises Accuracy at $4\times4$ from
$1.38\%$ (SymmetricDiffusers) to $3.00\%$ and Correct from $58.5\%$ to
$63.5\%$; at $5\times5$ and $6\times6$, where every method has essentially
zero Accuracy, Correct rises from $37.9\%$ to $53.2\%$ and from $18.5\%$ to
$41.1\%$, and Kendall-$\tau$ from $0.285$ to $0.438$ and from $0.121$ to
$0.323$. Seed-to-seed variation is at most $0.44$ points of Correct. In
the orbit arm, $96.1\%$ and $89.2\%$ of patches are in admissible slots at
$5\times5$ and $6\times6$.

Table~\ref{tab:jigsaw_rmse} gives PermFlow's reconstruction RMSE and MAE
at $\sigma=0.01$ on the full test split, averaged over three seeds. The
published values are on a different scale, several exceeding $1$ for
images in $[0,1]$, and are therefore not comparable.
Table~\ref{tab:jigsaw_steps} gives the full Euler-step sweep for the
seed-42 models, with one sample per image on the full test split: Valid
and content-based Correct at $\sigma=0$, and Correct and Kendall-$\tau$ at
$\sigma=0.01$.

\begin{table}[htbp]
\centering
\small
\caption{Jigsaw ($\sigma=0.01$): PermFlow reconstruction errors.}
\label{tab:jigsaw_rmse}
\begin{tabular}{lcc}
\toprule
Grid & RMSE & MAE \\
\midrule
$4{\times}4$ & 0.0090{\scriptsize$\pm$0.0002} & 0.0030{\scriptsize$\pm$0.0000} \\
$5{\times}5$ & 0.0269{\scriptsize$\pm$0.0006} & 0.0084{\scriptsize$\pm$0.0002} \\
$6{\times}6$ & 0.0875{\scriptsize$\pm$0.0040} & 0.0300{\scriptsize$\pm$0.0015} \\
\bottomrule
\end{tabular}
\end{table}

\begin{table}[t]
\centering
\scriptsize
\caption{Jigsaw: number of Euler steps $S$. One step is best at every grid.}
\label{tab:jigsaw_steps}
\setlength{\tabcolsep}{3pt}
\begin{tabular}{c|cccccc|cccccc}
\toprule
& \multicolumn{6}{c|}{$\sigma=0$ (orbit)} & \multicolumn{6}{c}{$\sigma=0.01$ (published setting)} \\
& \multicolumn{2}{c}{$4{\times}4$} & \multicolumn{2}{c}{$5{\times}5$} & \multicolumn{2}{c|}{$6{\times}6$} & \multicolumn{2}{c}{$4{\times}4$} & \multicolumn{2}{c}{$5{\times}5$} & \multicolumn{2}{c}{$6{\times}6$} \\
$S$ & Valid & Corr.(c) & Valid & Corr.(c) & Valid & Corr.(c) & Correct & $\tau$ & Correct & $\tau$ & Correct & $\tau$ \\
\midrule
1 & 93.5 & 99.0 & 78.1 & 96.9 & 46.9 & 91.4 & 63.78 & 0.5473 & 53.71 & 0.4421 & 42.81 & 0.3380 \\
2 & 93.2 & 98.9 & 76.1 & 96.6 & 42.5 & 90.3 & 63.66 & 0.5460 & 53.47 & 0.4409 & 42.03 & 0.3314 \\
5 & 92.9 & 98.9 & 73.6 & 96.2 & 41.5 & 89.9 & 63.61 & 0.5454 & 53.19 & 0.4376 & 41.52 & 0.3263 \\
10 & 92.7 & 98.8 & 72.5 & 96.1 & 40.8 & 89.5 & 63.56 & 0.5448 & 53.04 & 0.4366 & 41.16 & 0.3229 \\
20 & 92.6 & 98.8 & 72.2 & 96.0 & 40.2 & 89.3 & 63.54 & 0.5449 & 52.95 & 0.4353 & 40.95 & 0.3201 \\
50 & 92.6 & 98.8 & 71.9 & 95.9 & 39.8 & 89.2 & 63.52 & 0.5448 & 52.90 & 0.4348 & 40.76 & 0.3181 \\
\bottomrule
\end{tabular}
\end{table}

At $\sigma=0$ the benchmark's Accuracy, which scores against the one
labelled permutation, reports $3.2\%$ for PermFlow at $4\times4$, while
$92.5\%$ of its single samples reassemble the image exactly; every
reassembly drawn from a different orbit member is pixel-identical to the
original and still counted wrong. Orbit-aware metrics are therefore
necessary in this arm.

More Euler steps hurt on jigsaw puzzles (Table~\ref{tab:jigsaw_steps}).
The degradation is a property of the learned field, not of the
integrator: with the exact conditional velocity every $S$ reaches
$P^\star$ to within $10^{-7}$. Along the training path the network's
relative error is largest at $t=0$ ($0.43$ at $4\times4$, $0.78$ at
$6\times6$), where the state is pure noise, and at most $0.25$ for
$t\ge0.3$, where the state already contains $tP^\star$. One step uses only
the $t=0$ prediction, on a state drawn exactly from the training
distribution; with more steps, the first step moves toward a blurred
prediction, and later steps evaluate the field on states containing that
blurred matrix rather than a vertex, which it never saw in training. The
effect grows with $n$, as the $t=0$ prediction degrades; on sorting that
prediction is accurate and extra steps are harmless.

\subsection{The noise level as a knob}
Since $\sigma$ is continuous, the ambiguity can be swept rather than
toggled. Measured on the 10{,}000 test images, the fraction with
$|\mathcal{O}|>1$ is $99.5\%$, $99.5\%$, $63.7\%$, $0\%$, $0\%$ at $g=4$
for $\sigma\in\{0,10^{-4},10^{-3},3\times10^{-3},10^{-2}\}$, and $100\%$,
$100\%$, $99.8\%$, $0\%$, $0\%$ at $g=5$. Keeping the noised image in the
8-bit domain is what makes $\sigma$ a genuine knob: in float32 any
$\sigma>0$ destroys byte equality and $|\mathcal{O}|$ jumps from
$\prod_bm_b!$ to $1$. At $\sigma=0$ requantisation is the identity, and at
$\sigma=0.01$ about $84\%$ of pixels change, so $|\mathcal{O}|=1$ in both
endpoint arms. Table~\ref{tab:orbit_stats} gives the orbit statistics.

\begin{table}[htbp]
\centering
\small
\caption{Orbit statistics of the $\sigma=0$ arm (test split).}
\label{tab:orbit_stats}
\begin{tabular}{lccc}
\toprule
& $4{\times}4$ & $5{\times}5$ & $6{\times}6$ \\
\midrule
All-zero patches per image (mean) & 6.6 & 11.8 & 18.8 \\
Images with $|\mathcal{O}|>1$ & 99.3\% & 100\% & 100\% \\
Median $|\mathcal{O}|$ & $5{,}040$ & $4.8\times10^{8}$ & $6.4\times10^{15}$ \\
\bottomrule
\end{tabular}
\end{table}

\section{Conditional Mallows Mixtures and Sushi: Details}
\label{app:mallows}

\subsection{The conditional Mallows mixture}
A synthetic task is specified by the number of items $n$, the number of
components $M$, the dispersion $\theta$, and a random seed that fixes the
following quantities once: item features $E\in\R^{n\times d}$ with
i.i.d.\ standard normal entries ($d=8$), component maps
$A_m\in\R^{d\times d}$ with i.i.d.\ $\mathcal{N}(0,1/d)$ entries, and
mixture weights $U\in\R^{M\times d}$ with i.i.d.\ $\mathcal{N}(0,1/d)$
entries. A context is $\mathbf{x}\sim\mathcal{N}(0,I_d)$. Component $m$
has centre $\sigma_m(\mathbf{x})=\mathrm{argsort}(-EA_m\mathbf{x})$ and
weight $w_m(\mathbf{x})=\mathrm{softmax}(\beta U\mathbf{x})_m$ with
$\beta=2$, and
\[
p(\pi\mid\mathbf{x})=\sum_{m=1}^M w_m(\mathbf{x})\,
\frac{\exp(-\theta\,d_K(\pi,\sigma_m(\mathbf{x})))}{Z(\theta)},\qquad
Z(\theta)=\prod_{i=1}^{n}\frac{1-e^{-i\theta}}{1-e^{-\theta}}.
\]
In the \emph{reversed} setting ($M=2$) we set $A_2=-A_1$, so that
$\sigma_2(\mathbf{x})$ is the reversal of $\sigma_1(\mathbf{x})$ for every
context. We use $\theta=1$ throughout. Training samples are drawn exactly:
a component by its weight, then a ranking by the repeated insertion model,
which inserts the $i$-th item of the centre at position $j\in\{0,\dots,i\}$
with probability proportional to $e^{-\theta(i-j)}$. Training data are
generated afresh at every step, so there is no fixed training set. The
model input is one token per item, the concatenation of $E_i$ and
$\mathbf{x}$ (dimension $2d=16$).

We verified the sampler against the exact probabilities: for $n=5$, the
total-variation distance between $3\times10^5$ samples and $p(\cdot\mid\mathbf{x})$
is $0.007$, the Monte Carlo level.

\subsection{Sushi}
We use item set A of the Sushi preference data
\citep{kamishima2003nantonac}: 5{,}000 respondents each ranked the same 10
sushi from most to least preferred. A fixed random split assigns 4{,}000
respondents to training and 1{,}000 to testing. Each item token contains a
one-hot item identifier, the item attributes from the dataset (style,
major group, one-hot minor group, and four standardised numerical
attributes: oiliness, eating frequency, price, and availability), and the
respondent's attributes (gender, one-hot age group, one-hot region and
east/west indicator until age 15 and currently, and whether the two
differ), 62 dimensions in total. The task is conditional: generate
rankings for a held-out respondent from their attributes.

PermFlow uses the independent coupling. Besides the state token used
elsewhere, we evaluate a variant (``rows'') that also adds row $i$ of $X_t$
to the token of item $i$. Baselines share the item encoder and transformer:
GS trained on the one-hot observed ranking, Plackett--Luce (PL) trained by
maximum likelihood, and a mixture of four PL models. The PL mixture is a flexible parametric baseline and is the strongest
method on the synthetic mixtures. Models have 0.81--0.85M parameters and
three seeds.

\subsection{Models and training}
All models share an item MLP (two linear layers with GELU to width 128)
and a 4-layer pre-norm transformer (4 heads, feed-forward width 512,
dropout 0.1) over the item tokens.
\begin{itemize}
\item \textbf{PermFlow} adds the state token and time embedding of
  Section~\ref{sec:permflow} and the bilinear head; the \emph{rows} variant
  additionally adds a linear projection of row $i$ of $X_t$ to the token of
  item $i$. Independent coupling, $\sigma_0=1$, $S=10$ Euler steps unless
  stated. 0.84M parameters (0.85M on Sushi).
\item \textbf{GS} uses the same bilinear head as a score matrix and is
  trained with the squared loss between $\mathrm{Sinkhorn}((S+G)/\tau)$
  ($\tau=1$, 20 iterations) and the one-hot observed permutation; samples
  use the Gumbel-Matching decoder $\mathrm{Round}(S+G)$ (no noise:
  $\mathrm{Round}(S)$); Table~\ref{tab:mallows_decoders} compares the soft
  decoder $\mathrm{Round}(\mathrm{Sinkhorn}((S+G)/\tau))$. 0.81M parameters.
\item \textbf{PL} maps each item token to a scalar score and is trained by
  maximum likelihood; samples sort the Gumbel-perturbed scores. 0.81M.
\item \textbf{PL mixture} has $M'=4$ score heads and mixture logits from the
  mean-pooled tokens, trained by maximum likelihood. 0.81M. It is a flexible parametric baseline whose components can place mass on each centre, and it is the strongest method on the synthetic mixtures.
\end{itemize}
All models are trained with AdamW (learning rate $3\times10^{-4}$, weight
decay $10^{-4}$), cosine decay to $10^{-5}$, and gradient clipping at norm
1.0: for $4\times10^4$ steps of batch size 512 on the synthetic tasks, and
for $5{,}000$ steps of batch size 256 on Sushi (320 passes over the
training respondents). Three seeds per configuration, final-step model.

\subsection{Metrics}
On the synthetic tasks we fix 200 test contexts and draw $K=1000$ samples from each model per
context. For $n\le8$ we enumerate $S_n$ and compute the exact
$p(\cdot\mid\mathbf{x})$.
\begin{itemize}
\item \emph{TV}: total-variation distance between the empirical
  distribution of the $K$ samples and $p(\cdot\mid\mathbf{x})$ ($n\le8$).
  Because $K$ is smaller than $|S_n|$, even exact samples have nonzero TV;
  this floor is reported.
\item \emph{Mode TV}: each permutation is assigned to the component whose
  centre is nearest in Kendall distance; mode TV is the total-variation
  distance between the resulting $M$-way distribution under the model
  samples and under the exact distribution. For $M=2$ this is the
  absolute error in the probability of the first mode, the analogue of
  the calibration error of Section~\ref{sec:mnist}.
\item \emph{Pairwise error}: mean over item pairs of the absolute error in
  the probability that item $i$ is ranked above item $j$.
\item \emph{Off-mode excess}: the fraction of samples farther (in Kendall
  distance) from every centre than the exact $95\%$ quantile, minus the
  exact mass beyond that quantile.
\end{itemize}
For $n=10$, enumeration is replaced by $2\times10^4$ exact samples per
context, and full TV is not reported. Every metric is averaged over
contexts and is also computed for $K$ exact samples (the oracle floor).

On Sushi, for each held-out respondent we draw $K=100$ samples. Per respondent we
report the Kendall-$\tau$ correlation between a sample and the
respondent's own ranking, averaged over samples (``mean'') and maximised
over samples (``best of $K$''). Distributionally, we take one sample per
respondent, compute the population statistics over the 1{,}000 test
respondents, and compare them with the same statistics of the observed
test rankings: the distribution of the top item (TV), the item-by-position
marginal matrix (mean absolute error), the pairwise preference matrix
(mean absolute error), and the distribution of the ordered top three
(TV over $720$ outcomes); these are averaged over 20 independent draws.
The same statistics computed on the training rankings instead of model
samples give a data floor.

\subsection{Additional results}
Table~\ref{tab:mallows} reports means over three seeds with $K=1000$
samples per context; PermFlow integrates with $S=50$ Euler steps and GS
samples $\mathrm{Round}(S+G)$. Full TV, standard deviations and GS
without noise are in Table~\ref{tab:mallows_full}, which uses the same
settings ($\theta=1$, 200 test contexts) and whose first row is the
finite-sample floor of exact samples; exact TV is not computed for $n=10$.
In
Table~\ref{tab:sushi} (three seeds), PermFlow uses $S=10$ and GS samples
$\mathrm{Round}(S+G)$; the population statistics use one sample per
respondent and the Kendall-$\tau$ columns $K=100$ samples. The mean
Kendall-$\tau$ rewards concentrating on one ranking and is therefore not
bolded; standard deviations are in Table~\ref{tab:sushi_full}, which
uses the same settings.

On the reversed mixtures, PL places a large share of its samples between
the two modes: its off-mode excess is $0.16$--$0.46$ against an ideal of
$0$, averaged over three seeds, as shown in
Table~\ref{tab:mallows_offmode}. On Sushi, GS has the highest mean
Kendall-$\tau$ to the respondent's own ranking because it concentrates its
samples on one ranking, which the distributional metrics penalise.

\begin{table}[htbp]
\centering
\scriptsize
\caption{Conditional Mallows mixtures, all metrics with standard deviations
(lower is better). First row: floor of exact samples.}
\label{tab:mallows_full}
\setlength{\tabcolsep}{2pt}
\resizebox{\linewidth}{!}{%
\begin{tabular}{l|ccc|ccc|ccc|cc}
\toprule
& \multicolumn{3}{c|}{$n{=}6$, 2 reversed modes} & \multicolumn{3}{c|}{$n{=}8$, 2 reversed modes} & \multicolumn{3}{c|}{$n{=}8$, 3 random modes} & \multicolumn{2}{c}{$n{=}10$, 2 reversed} \\
Method & TV & Mode & Pair & TV & Mode & Pair & TV & Mode & Pair & Mode & Pair \\
\midrule
\textit{Exact samples (floor)} & 0.196{\scriptsize$\pm$0.000} & 0.008{\scriptsize$\pm$0.000} & 0.011{\scriptsize$\pm$0.000} & 0.443{\scriptsize$\pm$0.000} & 0.008{\scriptsize$\pm$0.000} & 0.010{\scriptsize$\pm$0.000} & 0.464{\scriptsize$\pm$0.000} & 0.012{\scriptsize$\pm$0.000} & 0.010{\scriptsize$\pm$0.000} & 0.009{\scriptsize$\pm$0.000} & 0.011{\scriptsize$\pm$0.000} \\
\midrule
Plackett--Luce & 0.439{\scriptsize$\pm$0.001} & 0.066{\scriptsize$\pm$0.001} & 0.039{\scriptsize$\pm$0.000} & 0.708{\scriptsize$\pm$0.001} & 0.094{\scriptsize$\pm$0.001} & 0.047{\scriptsize$\pm$0.000} & 0.712{\scriptsize$\pm$0.001} & 0.093{\scriptsize$\pm$0.001} & 0.056{\scriptsize$\pm$0.000} & 0.132{\scriptsize$\pm$0.001} & 0.060{\scriptsize$\pm$0.001} \\
GS, no noise & 0.897{\scriptsize$\pm$0.001} & 0.166{\scriptsize$\pm$0.000} & 0.267{\scriptsize$\pm$0.000} & 0.960{\scriptsize$\pm$0.000} & 0.167{\scriptsize$\pm$0.000} & 0.250{\scriptsize$\pm$0.000} & 0.966{\scriptsize$\pm$0.001} & 0.243{\scriptsize$\pm$0.000} & 0.219{\scriptsize$\pm$0.000} & 0.190{\scriptsize$\pm$0.000} & 0.255{\scriptsize$\pm$0.000} \\
GS & 0.491{\scriptsize$\pm$0.001} & 0.096{\scriptsize$\pm$0.001} & 0.133{\scriptsize$\pm$0.001} & 0.636{\scriptsize$\pm$0.001} & 0.103{\scriptsize$\pm$0.001} & 0.118{\scriptsize$\pm$0.000} & 0.642{\scriptsize$\pm$0.001} & 0.136{\scriptsize$\pm$0.003} & 0.106{\scriptsize$\pm$0.001} & 0.135{\scriptsize$\pm$0.001} & 0.120{\scriptsize$\pm$0.001} \\
\midrule
PermFlow & 0.328{\scriptsize$\pm$0.019} & 0.025{\scriptsize$\pm$0.007} & 0.043{\scriptsize$\pm$0.005} & 0.569{\scriptsize$\pm$0.020} & 0.041{\scriptsize$\pm$0.006} & 0.054{\scriptsize$\pm$0.004} & 0.721{\scriptsize$\pm$0.009} & 0.106{\scriptsize$\pm$0.003} & 0.083{\scriptsize$\pm$0.002} & 0.038{\scriptsize$\pm$0.010} & 0.048{\scriptsize$\pm$0.006} \\
PermFlow (rows) & 0.295{\scriptsize$\pm$0.007} & 0.017{\scriptsize$\pm$0.005} & 0.033{\scriptsize$\pm$0.004} & 0.526{\scriptsize$\pm$0.003} & 0.029{\scriptsize$\pm$0.006} & 0.042{\scriptsize$\pm$0.004} & 0.635{\scriptsize$\pm$0.005} & 0.051{\scriptsize$\pm$0.002} & 0.056{\scriptsize$\pm$0.001} & 0.031{\scriptsize$\pm$0.001} & 0.040{\scriptsize$\pm$0.001} \\
\midrule
PL mixture ($M'{=}4$) & \textbf{0.245{\scriptsize$\pm$0.002}} & \textbf{0.011{\scriptsize$\pm$0.000}} & \textbf{0.016{\scriptsize$\pm$0.000}} & \textbf{0.489{\scriptsize$\pm$0.001}} & \textbf{0.010{\scriptsize$\pm$0.001}} & \textbf{0.016{\scriptsize$\pm$0.000}} & \textbf{0.545{\scriptsize$\pm$0.002}} & \textbf{0.017{\scriptsize$\pm$0.000}} & \textbf{0.027{\scriptsize$\pm$0.000}} & \textbf{0.013{\scriptsize$\pm$0.002}} & \textbf{0.017{\scriptsize$\pm$0.001}} \\
\bottomrule
\end{tabular}}
\end{table}

\begin{table}[htbp]
\centering
\scriptsize
\caption{Sushi, all methods with standard deviations. Bold: best in column,
excluding the mean $\tau$.}
\label{tab:sushi_full}
\setlength{\tabcolsep}{3pt}
\begin{tabular}{lcccc|cc}
\toprule
Method & Top-1 TV & Position MAE & Pairwise MAE & Top-3 TV & $\tau$ mean & $\tau$ best \\
\midrule
\textit{Training rankings (data floor)} & 0.035 & 0.0091 & 0.012 & 0.295 & -- & -- \\
\midrule
Plackett--Luce & 0.169{\scriptsize$\pm$0.003} & 0.0216{\scriptsize$\pm$0.0004} & 0.038{\scriptsize$\pm$0.001} & 0.479{\scriptsize$\pm$0.006} & 0.149{\scriptsize$\pm$0.000} & 0.648{\scriptsize$\pm$0.000} \\
GS, no noise & 0.365{\scriptsize$\pm$0.020} & 0.0597{\scriptsize$\pm$0.0016} & 0.145{\scriptsize$\pm$0.005} & 0.628{\scriptsize$\pm$0.010} & 0.274{\scriptsize$\pm$0.003} & 0.274{\scriptsize$\pm$0.003} \\
GS & 0.182{\scriptsize$\pm$0.003} & 0.0238{\scriptsize$\pm$0.0002} & 0.058{\scriptsize$\pm$0.001} & 0.417{\scriptsize$\pm$0.002} & 0.215{\scriptsize$\pm$0.001} & 0.650{\scriptsize$\pm$0.004} \\
PermFlow & 0.060{\scriptsize$\pm$0.003} & 0.0124{\scriptsize$\pm$0.0001} & \textbf{0.019{\scriptsize$\pm$0.001}} & 0.390{\scriptsize$\pm$0.003} & 0.148{\scriptsize$\pm$0.002} & \textbf{0.688{\scriptsize$\pm$0.001}} \\
PermFlow (rows) & \textbf{0.051{\scriptsize$\pm$0.005}} & \textbf{0.0115{\scriptsize$\pm$0.0002}} & 0.022{\scriptsize$\pm$0.000} & 0.434{\scriptsize$\pm$0.003} & 0.144{\scriptsize$\pm$0.001} & 0.676{\scriptsize$\pm$0.002} \\
PL mixture ($M'{=}4$) & 0.070{\scriptsize$\pm$0.005} & 0.0131{\scriptsize$\pm$0.0000} & 0.021{\scriptsize$\pm$0.000} & \textbf{0.382{\scriptsize$\pm$0.003}} & 0.155{\scriptsize$\pm$0.001} & 0.667{\scriptsize$\pm$0.001} \\
\bottomrule
\end{tabular}
\end{table}

\begin{table}[htbp]
\centering
\scriptsize
\caption{Off-mode excess on the Mallows mixtures (ideal $0$). PL places a large
share of its samples between the modes.}
\label{tab:mallows_offmode}
\begin{tabular}{lcccc}
\toprule
Method & $n{=}6$, $M{=}2$ & $n{=}8$, $M{=}2$ & $n{=}8$, $M{=}3$ & $n{=}10$, $M{=}2$ \\
\midrule
\textit{Exact samples (floor)} & -0.000{\scriptsize$\pm$0.000} & -0.000{\scriptsize$\pm$0.000} & -0.000{\scriptsize$\pm$0.000} & -0.001{\scriptsize$\pm$0.000} \\
Plackett--Luce & 0.161{\scriptsize$\pm$0.002} & 0.306{\scriptsize$\pm$0.002} & 0.235{\scriptsize$\pm$0.002} & 0.455{\scriptsize$\pm$0.004} \\
GS, no noise & -0.044{\scriptsize$\pm$0.002} & -0.043{\scriptsize$\pm$0.005} & -0.037{\scriptsize$\pm$0.000} & -0.038{\scriptsize$\pm$0.004} \\
GS & 0.022{\scriptsize$\pm$0.001} & 0.098{\scriptsize$\pm$0.001} & 0.046{\scriptsize$\pm$0.002} & 0.176{\scriptsize$\pm$0.003} \\
PermFlow & 0.040{\scriptsize$\pm$0.011} & 0.048{\scriptsize$\pm$0.020} & 0.200{\scriptsize$\pm$0.016} & 0.071{\scriptsize$\pm$0.016} \\
PermFlow (rows) & 0.035{\scriptsize$\pm$0.004} & 0.030{\scriptsize$\pm$0.002} & 0.101{\scriptsize$\pm$0.007} & 0.034{\scriptsize$\pm$0.002} \\
PL mixture ($M'{=}4$) & 0.020{\scriptsize$\pm$0.003} & 0.042{\scriptsize$\pm$0.001} & 0.039{\scriptsize$\pm$0.001} & 0.054{\scriptsize$\pm$0.002} \\
\bottomrule
\end{tabular}
\end{table}

\begin{table}[t]
\centering
\scriptsize
\caption{Number of Euler steps $S$ in the sample-valued regime. Unlike in the
set-valued tasks, one step is far from the optimum.}
\label{tab:mallows_steps}
\setlength{\tabcolsep}{3pt}
\begin{tabular}{c|cc|cc|cc}
\toprule
& \multicolumn{2}{c|}{$n{=}8$, 2 reversed modes} & \multicolumn{2}{c|}{$n{=}8$, 3 random modes} & \multicolumn{2}{c}{Sushi (PermFlow)} \\
$S$ & PermFlow & PermFlow (rows) & PermFlow & PermFlow (rows) & Top-1 TV & Pairwise MAE \\
\midrule
1 & 0.154{\scriptsize$\pm$0.005} & 0.149{\scriptsize$\pm$0.002} & 0.227{\scriptsize$\pm$0.003} & 0.194{\scriptsize$\pm$0.007} & 0.395{\scriptsize$\pm$0.023} & 0.104{\scriptsize$\pm$0.005} \\
2 & 0.131{\scriptsize$\pm$0.008} & 0.122{\scriptsize$\pm$0.007} & 0.148{\scriptsize$\pm$0.007} & 0.124{\scriptsize$\pm$0.007} & 0.132{\scriptsize$\pm$0.016} & 0.042{\scriptsize$\pm$0.005} \\
5 & 0.085{\scriptsize$\pm$0.008} & 0.067{\scriptsize$\pm$0.006} & 0.120{\scriptsize$\pm$0.004} & 0.073{\scriptsize$\pm$0.004} & 0.061{\scriptsize$\pm$0.005} & 0.021{\scriptsize$\pm$0.001} \\
10 & 0.066{\scriptsize$\pm$0.007} & 0.050{\scriptsize$\pm$0.006} & 0.112{\scriptsize$\pm$0.002} & 0.059{\scriptsize$\pm$0.003} & 0.060{\scriptsize$\pm$0.003} & 0.019{\scriptsize$\pm$0.001} \\
20 & 0.051{\scriptsize$\pm$0.006} & 0.037{\scriptsize$\pm$0.005} & 0.108{\scriptsize$\pm$0.002} & 0.053{\scriptsize$\pm$0.002} & 0.063{\scriptsize$\pm$0.007} & 0.019{\scriptsize$\pm$0.000} \\
50 & 0.041{\scriptsize$\pm$0.006} & 0.029{\scriptsize$\pm$0.006} & 0.106{\scriptsize$\pm$0.003} & 0.051{\scriptsize$\pm$0.002} & 0.060{\scriptsize$\pm$0.010} & 0.019{\scriptsize$\pm$0.000} \\
\bottomrule
\end{tabular}
\end{table}

Table~\ref{tab:mallows_decoders} evaluates GS with both decoders on the
same trained models (three seeds). On these
tasks the soft decoder and the Gumbel-Matching decoder agree to within
$0.003$ on every metric with noise; without noise they differ slightly on
Sushi (top-1 TV $0.365$ versus $0.389$).

\begin{table}[htbp]
\centering
\scriptsize
\caption{GS on the ranking tasks with the Gumbel-Matching and the soft decoder,
which agree closely.}
\label{tab:mallows_decoders}
\begin{tabular}{llccc}
\toprule
Task & Decoder & \multicolumn{3}{c}{Metrics (lower is better)} \\
\midrule
$n{=}6$, 2 rev. & Gumbel-Matching & tv: 0.492 & mode tv: 0.096 & pair mae: 0.133 \\
 & soft & tv: 0.491 & mode tv: 0.096 & pair mae: 0.132 \\
 & GM, no noise & tv: 0.897 & mode tv: 0.166 & pair mae: 0.267 \\
 & soft, no noise & tv: 0.897 & mode tv: 0.166 & pair mae: 0.267 \\
\midrule
$n{=}8$, 2 rev. & Gumbel-Matching & tv: 0.636 & mode tv: 0.103 & pair mae: 0.118 \\
 & soft & tv: 0.635 & mode tv: 0.103 & pair mae: 0.116 \\
 & GM, no noise & tv: 0.960 & mode tv: 0.167 & pair mae: 0.250 \\
 & soft, no noise & tv: 0.960 & mode tv: 0.167 & pair mae: 0.250 \\
\midrule
$n{=}8$, 3 rand. & Gumbel-Matching & tv: 0.643 & mode tv: 0.135 & pair mae: 0.105 \\
 & soft & tv: 0.642 & mode tv: 0.135 & pair mae: 0.104 \\
 & GM, no noise & tv: 0.966 & mode tv: 0.243 & pair mae: 0.219 \\
 & soft, no noise & tv: 0.966 & mode tv: 0.243 & pair mae: 0.219 \\
\midrule
$n{=}10$, 2 rev. & Gumbel-Matching & mode tv: 0.135 & pair mae: 0.119 \\
 & soft & mode tv: 0.135 & pair mae: 0.117 \\
 & GM, no noise & mode tv: 0.190 & pair mae: 0.255 \\
 & soft, no noise & mode tv: 0.190 & pair mae: 0.255 \\
\midrule
Sushi & Gumbel-Matching & top1 tv: 0.183 & pair mae: 0.057 & top3 tv: 0.415 \\
 & soft & top1 tv: 0.186 & pair mae: 0.057 & top3 tv: 0.418 \\
 & GM, no noise & top1 tv: 0.365 & pair mae: 0.145 & top3 tv: 0.628 \\
 & soft, no noise & top1 tv: 0.389 & pair mae: 0.145 & top3 tv: 0.626 \\
\bottomrule
\end{tabular}
\end{table}

\section{Integration Steps: Protocol}
\label{app:steps}
Table~\ref{tab:steps} takes models that have already been trained and
only changes the number of Euler steps $S$ used at inference; no model is
retrained. Each column summarises one task by a single metric.

For the set-valued tasks, the sorting columns give the calibration error at
$K=100$ for PermFlow with the nearest and with the calibrated coupling,
averaged over three seeds; Table~\ref{tab:mnist_steps} reports coverage and
calibration error at all $K$. The jigsaw columns give the single-sample
validity at $\sigma=0$ on the $4\times4$ and $6\times6$ grids and the
patch-level Correct at $\sigma=0.01$ on the $6\times6$ grid, for the
seed-42 models; Table~\ref{tab:jigsaw_steps} reports all grids and metrics.

For the sample-valued tasks, the Mallows columns give the mode TV of the
rows variant of PermFlow on the $n=8$ mixtures with two reversed or three
random centres, and the Sushi column the top-1 TV of PermFlow, all averaged
over three seeds. Table~\ref{tab:mallows_steps} additionally reports the
state-token variant on the Mallows mixtures and the pairwise error on
Sushi.

\end{document}